\documentclass[final,onefignum,onetabnum]{siamart220329}

\usepackage{graphicx,float}
\usepackage{subfigure}
\usepackage{amsfonts}

\newcommand{\floor}[1]{\lfloor #1 \rfloor}

\usepackage{algorithm}
\usepackage{algorithmic}
\usepackage{setspace}
\usepackage{float}

\newsiamremark{remark}{Remark}
\newtheorem{assumption}[theorem]{Assumption}

\usepackage{hyperref}
\hypersetup{
colorlinks=true, 
breaklinks=true, 
urlcolor= blue, 
linkcolor= red, 
citecolor=blue, 
pdftitle={}, 
pdfauthor={}, 
pdfsubject={} 
}

\emergencystretch=\maxdimen
\newcommand{\rev}[1]{\textcolor{black}{#1}}

\usepackage{comment}

\title{Logarithmic regret bounds for continuous-time average-reward Markov decision processes }

\author{Xue feng
			Gao\thanks{Department of Systems
			Engineering and Engineering Management, The Chinese University of Hong Kong, Hong Kong, China.
			(\email{xfgao@se.cuhk.edu.hk}).}
	 \and	Xun Yu Zhou\thanks{Department of Industrial Engineering and Operations Research and The Data Science Institute, Columbia University, New York, NY 10027, USA. (\email{xz2574@columbia.edu}). }
		}

\begin{document}

\maketitle

\begin{abstract}
We consider reinforcement learning for continuous-time Markov decision processes (MDPs) in the infinite-horizon, average-reward setting.
In contrast to discrete-time MDPs,  a continuous-time process moves to a state and  stays there for a random holding time after an action is taken. With unknown transition probabilities and rates of exponential holding times,
we derive instance-dependent regret lower bounds that are logarithmic in the time horizon. Moreover, we design a learning algorithm
and establish a finite-time regret bound that achieves the logarithmic growth rate. Our analysis builds upon upper confidence reinforcement learning, a delicate estimation of the mean holding times, and stochastic comparison of point processes.

\end{abstract}




\section{Introduction}

Reinforcement learning (RL) is the problem of an agent learning how to map states to actions in order to maximize the reward over time in an unknown environment. It has received significant attention in the past decades, and the key challenge is in balancing the trade-off between exploration and exploitation \cite{sutton2018reinforcement}.
The common model for RL is a Markov Decision Process (MDP), which provides a mathematical framework for modeling sequential decision making problems under uncertainty.
Most of the current studies
on RL focus on developing algorithms and analysis for \textit{discrete-time} MDPs. In contrast, less attention has been paid to  \textit{continuous-time} MDPs. However, there are many real-world applications where one {needs} to consider continuous-time MDPs.  {Examples include control of queueing
systems, control of infectious diseases, preventive maintenance and high frequency trading;  see, e.g., \cite{guo2009continuous2, piunovskiy2020continuous}, Chapter 11 of \cite{puterman2014markov} and the references therein. One may propose discretizing time upfront to turn a continuous-time control problem into a  discrete-time one and then apply the
existing results and algorithms. However, it is well known in the RL community that this
approach is very sensitive to time discretization and may perform poorly with small time steps; see e.g. \cite{tallec2019making}. In this paper, we study the learning of continuous-time MDPs (CTMDPs), carrying out theoretical analysis entirely in continuous time. We aim to understand fundamental performance limits of learning algorithms as well as to develop learning algorithms with theoretical guarantees.}

In RL, there are typically three different settings/criteria: (1) infinite-horizon discounted reward setting, (2) infinite-horizon average-reward (ergodic) setting, and (3) finite-horizon episodic setting. In this paper, we focus on the infinite-horizon average-reward setting, where the agent aims to learn from data (e.g., states, actions and rewards) a policy that optimizes the long-run average reward.
This setting has been heavily studied for learning in discrete-time MDPs; see e.g. \cite{bartlett2009regal, filippi2010optimism, fruit2018efficient, Jaksch2010, ortner2007logarithmic}. More importantly, for many CTMDPs such as the optimal control of queueing systems and their applications, it is useful to study performance measures such as the average number of customers in the system and optimize the long-run average reward/cost as the systems often run for a long time; see e.g. \cite{dai2020queueing}.

In this paper, we focus on theoretical analysis and algorithm design for learning in tabular CTMDPs with finite state space $\mathcal{S}$ and finite action space $\mathcal{A}$. Upon arriving at a state $s$, the agent takes an action $a$ and receives some reward. The system remains in state $s$ for a random holding time which is \textit{exponentially} distributed with some unknown rate {that depends on the state-action pair $(s,a)$. Then it jumps to another state with some unknown transition probability, at which another action is made. This series of events is repeated.} The performance of the agent's algorithm is often measured by the total \textit{regret} after some time period,
which is defined as the difference between the rewards collected by the algorithm during learning and the rewards of an optimal policy should the model parameters be completely known. If a learning algorithm achieves {\it sublinear} regret (in terms of the time horizon), then the average reward of the algorithm converges to the optimal average reward per unit time, while the convergence speed is determined by the specific growth rate of the regret.

While computational RL methods have been developed for CTMDPs as well as  the more general Semi-Markov Decision Processes (SMDPs) in the earlier literature \cite{bradtke1995reinforcement, das1999solving}, available theoretical results on regret bounds for continuous-time RL \rev{with discrete state spaces} are very limited. 
\rev{\cite{gao2022square} study RL for CTMDPs in
the finite-horizon episodic setting. They establish worst-case regret upper and lower bounds, both of the order of square-root on the number of episodes.}
\cite{fruit2017exploration} study learning in continuous-time (infinite-horizon) average-reward SMDPs which is more general than CTMDPs in that the holding times can follow general distributions. They adapt the UCRL2 algorithm by \cite{Jaksch2010} and show their algorithm achieves $O( \sqrt{ n})$ regret after $n$ decision steps (ignoring logarithmic factors and hiding dependency on other constants) in the infinite-horizon average reward setting. They also establish a minimax (worst-case) lower bound, in the sense  that for any algorithm there exists an SMDP such that the expected regret is of $O(\sqrt{n})$.

\rev{These $\sqrt{n}-$type regret bounds are {\it worst-case} performance bounds without exploiting
any special
structure of the underlying CTMDPs; hence they tend to be overly pessimistic about the complexity of the learning problem \cite{simchowitz2019non}. This motivates us to seek tighter regret bounds that are {\it instance-dependent}.}
The goal of the present paper is two fold. First, we
derive instance-dependent regret lower bounds that are logarithmic in the time horizon $T$
satisfied by {\it any} {``provably-efficient"} algorithm for learning average-reward CTMDPs. This provides the corresponding fundamental performance limits of learning algorithms.
Second, we design learning algorithms that indeed achieve the logarithmic rate. 
{While instance-dependent logarithmic regret bounds have been studied for learning discrete-time MDPs in the average-reward setting  (see, e.g., \cite{burnetas1997optimal, Jaksch2010, ok2018exploration}), there are no such results for the continuous-time average-reward RL to our best knowledge.  In this paper we fill this gap. Importantly, an extension from discrete-time to continuous-time incurs substantial difficulties and calls for a significantly different and involved analysis (see subsequent discussions for details). }

Specifically, we first establish (asymptotic) instance-dependent logarithmic regret lower bound for learning average-reward CTMDPs.
Since the set of decision epochs of CTMDPs are still discrete, we first prove a regret lower bound that is logarithmic in the number of decision steps.  {While the central  idea relies on a  change-of-measure argument similar to that in \cite{burnetas1997optimal}, considerable difficulties arise in applying the argument to the continuous-time case. First, in contrast with the discrete-time setting, the observations/data now contain the sequence of exponentially distributed holding times. Hence, the log-likelihood ratio of the observations under two different CTMDP models (with different holding time rates and transition probabilities), which is needed in the change-of-measure argument, is more complex to analyze compared with the one for discrete-time models  in \cite{burnetas1997optimal}. We overcome this difficulty by carefully exploiting the conditional independence of the holding time and the next state given the current state--action pair in CTMDPs. Second, the derivation of instance-dependent regret lower bounds focuses on so-called uniformly fast convergent (or uniformly good) algorithms whose definition in our setting is based on the regret defined in continuous time (see Equations~\eqref{eq:UF} and \eqref{eq:regret_T}). When we derive the regret lower bound in terms of the number of decision steps of a CTMDP, we need to show such algorithms are also uniformly fast convergent when we consider the regret in {\it discrete} decision epochs. Since the holding times of a CTMDP are random and not identically distributed in general, we apply stochastic comparisons of point processes \cite{whitt1981comparing} to overcome this difficulty. This technique also allows us to convert the regret lower bound in number of decisions to a logarithmic lower bound in the time horizon $T$. }

For regret upper bound and algorithm design, we propose the CT-UCRL (continuous-time upper confidence reinforcement learning) algorithm for learning CTMDP and establish a finite-time instance-dependent (or gap-dependent) upper bound for its regret. Our algorithm is a variant of the UCRL2 algorithm in \cite{Jaksch2010} for learning discrete-time MDPs although, again,  we need to deal with the additional holding times inherent in CTMDPs.
Our alogrithm is also related to the SMDP-UCRL algorithm for learning SMDPs in \cite{fruit2017exploration}. The main difference in the algorithm design is that we estimate the mean holding times with more refined estimators with tighter confidence bounds than the empirical mean  in \cite{fruit2017exploration}. Note that the holding times follow exponential distributions, which are \textit{unbounded and have heavier tails} than Gaussian distributions. As a consequence, the empirical mean holding times have exponential upper tails, leading to difficulties in the proof of logarithmic regret guarantees of learning algorithms, as is known in the multi-armed-bandit literature \cite{bubeck2013bandits}. {By using a refined estimator for the mean holding times,
we  prove that the CT-UCRL algorithm has a finite-time regret bound of $O(\frac{\log n}{g})$ after $n$ decision steps (ignoring dependance on other constants), where $g$ is an instance-dependent gap parameter. We also show that the CT-UCRL algorithm has a finite-time regret upper bound that is logarithmic in the time horizon $T$, again by using stochastic comparisons of point processes. }

{This result has two immediate implications. First, in our study of asymptotic instance-dependent regret lower bound, {\it a priori} it is not clear whether or not a uniformly fast convergent algorithm exists in learning average-reward CTMDPs. This result provides an affirmative answer to this question. Second, it implies that the logarithmic rate in the lower bound can indeed be achieved by our algorithm. While our algorithm may not be
asymptotically instance optimal (i.e., the regret upper bound matches the asymptotic regret lower bound as the horizon $T$ goes to infinity), it has a finite-time performance bound which is desirable in practice. This is in contrast with \cite{burnetas1997optimal} which designs an asymptotically optimal algorithm whose finite-time performance is however unclear in the discrete-time setting. }

{To prove the logarithmic regret upper bound of our proposed algorithm, we extend the approach of \cite{Jaksch2010} to the CTMDP setting. Once again, there are significant differences in our analysis compared with the discrete-time setting, due to exponential holding times and different Bellman optimality equations. In particular, we need several new ingredients in our regret analysis. First, we need a high probability confidence bound for mean holding times based on our estimator. Second, we need to analyze the solution procedure for finding the optimisitic CTMDP in the confidence region (see Equation~\eqref{eq:optimistic-model} in our algorithm). Finally, the regret incurred in each decision period depends on the length of the random holding time at that period, giving rise to some subtle difficulty in the regret analysis. To get around, we analyze an ``adjusted regret" corresponding to the case where the random holding time is replaced by its mean, which coincides with the actual regret in expectation. Such a technique is not required in the discrete-time setting. }

\rev{While our study focues on CTMDP with discrete state spaces, we conclude this introduction by mentioning a growing body of literature on RL for diffusion processes with continuous state spaces. A series of papers \cite{ JZ21, JZ22, jia2023q, Wang20} have formulated  RL for general diffusion processes and developed a martingale theoretical foundation for designing various  actor--critic and q-learning algorithms. \cite{Wangboyu2023} study RL for the continous-time optimal execution problem in finance and develop a finite-time convergence analysis of their offline actor-critic algorithm. In terms of regret bounds,
several recent papers \cite{basei2022logarithmic, guo2021reinforcement, szpruch2021exploration, szpruch2024optimal} study continuous-time RL for linear--quadratic/convex models and propose algorithms with sublinear regret upper bounds in the finite-horizon episodic setting.
For the infinite-horizon average reward setting, \cite{faradonbeh2023online, shirani2022thompson} study continuous-time linear-quadratic RL problems and propose algorithms with a square-root of time regret.
}

The remainder of the paper is organized as follows. In Section~\ref{sec:learnCTMDP}, we formulate the problem of learning in CTMDPs. In Section~\ref{sec:lowerbound} we present the result on the asymptotic regret lower bound, while in Section~\ref{sec:upperbound} we devise the CT-UCRL algorithm and establish a finite-time instance-dependent upper bound for its regret. Finally Section~\ref{sec:conclusion} concludes.  


\section{Formulation of Learning in CTMDPs} \label{sec:learnCTMDP}

We consider a CTMDP  with a finite state space $\mathcal{S}$ and a finite action space $\mathcal{A}$. \rev{We denote by $S$ and $A$ the cardinalities  of $\mathcal{S}$ and $\mathcal{A}$, respectively.}
Given a (deterministic, stationary, and Markov) {\it policy} $\pi$, which  is a map from $\mathcal{S}$ to $\mathcal{A}$, the process  moves as follows (see e.g. \cite{puterman2014markov,  serfozo1979equivalence}). At time $0,$ the system is at state $s_0 \in \mathcal{S} $ and the agent chooses an action $a_0=\pi(s_0) \in \mathcal{A}$ with a reward $r(s_0,a_0) \in [0,1]$  received as a consequence of the action. The system remains in state $s_0$ for a random holding time period $\tau_0$ that follows an exponential distribution with parameter $\lambda(s_0, a_0).$ Then it jumps to state $s_1 \in \mathcal{S}$ with transition probability $p(s_1 | s_0, a_0)$ at which  another action $a_1=\pi(s_1)$ is  made. This series of events is repeated indefinitely.   {Note that if action $a$ is chosen in state $s$, then the joint probability that the holding time in state $s$ is not greater than $t $ and the next state is $j$ is given by $ \left(1- e^{-\lambda(s, a) t}\right) \cdot p (j | s,a)$. }

A history $\omega_n$ of the CTMDP up to the $n$th decision epoch is denoted by $\omega_n \equiv (s_0, a_0, \tau_0, s_1, a_1,\tau_1 \cdots, s_{n-1}, a_{n-1},\tau_{n-1}, s_n)$, which is any feasible sequence of states, actions and holding times up to the $n$th decision epoch. Here, a major difference from  discrete-time MDPs is that the history now contains the holding times $\tau_0, \tau_1, \cdots,\tau_{n-1}$. Denote by $N(t)$ the total number of actions made up to (and including) time $t>0$.

We consider learning in an {\it unknown} CTMDP where the rates of exponential holding times $(\lambda(s,a))_{s \in \mathcal{S}, a \in \mathcal{A}}$ and
the transition probabilities $(p(\cdot |s, a))_{s \in \mathcal{S}, a \in \mathcal{A}}$ are unknown.
For simplicity we assume the functional form of the reward function $(r(s,a)))_{s \in \mathcal{S}, a \in \mathcal{A}}$ is known.\footnote{This assumption is not restrictive for the purpose of obtaining logarithmic regret bounds. For  lower bounds, it  can be relaxed following the discussion in Section~7 of  \cite{burnetas1997optimal} and \cite{ok2018exploration}. For  upper bounds, our proposed algorithm can be easily extended to incorporate the empirical estimates of the reward function following \cite{Jaksch2010}, without changing the logarithmic growth rate of the regret.} \rev{A {\it learning algorithm}, denoted by $\mathcal{G}$,  at each decision epoch $n \in \mathbb{N}$ generates a probability measure $\pi_n$ on $\mathcal{A}$ such that $\pi_n(\mathcal{A}| \omega_n) =1$, where $\pi_n$ is a possibly randomized law of selecting actions based on the entire history $\omega_n$. }
The goal is to minimize the expected regret to be defined shortly.

 Denote by $\mathcal{M} = ( p(\cdot |s, a) , \lambda(s,a))_{s \in \mathcal{S}, a \in \mathcal{A}} $ the unknown parameters of the CTMDP. For notational simplicity, if a CTMDP is described by the quadruplet $(\mathcal{S}, \mathcal{A}, \mathcal{M}, r)$ where $r=(r(s,a))_{s \in \mathcal{S}, a \in \mathcal{A}}$, it is referred to  as the CTMDP $\mathcal{M}$.\footnote{For simplicity, we only consider the `lump sum reward' $r(s,a)$ in the CTMDP. More generally, one can also incorporate the `running reward/cost', see Chapter 11 of \cite{puterman2014markov} for more details.}

We make two assumptions on the CTMDP $\mathcal{M}$ throughout the paper. The first one is on the unknown transition probability matrix, denoted by $P=( p(j|s, a) )_{s, j \in \mathcal{S}, a \in \mathcal{A}}$.
We are first given a family of  sets $\mathcal{S}^+(s,a)$ for all $s \in \mathcal{S}, a \in \mathcal{A}$. These sets are assumed to be known and they are independent of $\mathcal{M}$. 
Now define
{\small{
\begin{align*}
\Theta(s,a) = \left\{ q \in \mathbb{R}^{|\mathcal{S}|}:  \sum_{y \in \mathcal{S}} q(y)=1,  \text{$q(y)>0$, $ \forall \ y \in \mathcal{S}^+(s,a)$ and $q(y)=0,$ $\forall \ y \notin \mathcal{S}^+(s,a)$} \right\}.
\end{align*}
} }
Following  \cite{burnetas1997optimal} (see Assumption (A) therein), we make the following assumption on the transition probability matrix $P$. {While this assumption is fairly strong, it is needed in deriving explicit instance-dependent regret lower bounds in the literature on discrete time MDPs. Hence we inherit  such an assumption in our continuous-time setting.}

\begin{assumption} \label{assume-transition}
The transition probability vectors $p(\cdot| s, a) \in \Theta(s,a)$ for all $s \in \mathcal{S}, a \in \mathcal{A}$. In addition, the transition matrices $[p(j|s, \pi(s))]_{ s, j \in \mathcal{S} }$ are irreducible for all deterministic, stationary and Markov policies $\pi.$
\end{assumption}

Note that $p(\cdot| s, a) \in \Theta(s,a)$ implies that the support of the transition probability vector $p_s(a)$ is given by the known set $\mathcal{S}^+(s,a).$ \rev{The assumption that $\mathcal{S}^+(s,a)$ is known and independent of $\mathcal{M}$ is needed in proving our instance-dependent regret lower bound because the proof relies on change-of-measure arguments. In particular, to apply such arguments one needs to know $\mathcal{S}^+(s,a)$, the support of the next-state distribution, to study the Kullback--Leibler (KL) divergence between the transition probability distribution of the CTMDP $\mathcal{M}$  and that of an alternative CTMDP; see \eqref{eq:KLdiff}. This assumption of knowing $\mathcal{S}^+(s,a)$, however, is not needed in our algorithm design and the proof of a regret upper bound. For an example of CTMDPs satisfying this assumption,  consider the classical problem of optimal customer selection in a $M/M/n$ queue with $m$ customer classes in \cite{miller1969queueing}. It is formulated as a CTMDP with
state space $\mathcal{S} =\{0,1, \ldots, n\}$, where state $i$ means $i$ servers are free. For $i \ge 1,$ the possible action $a$ at state $i$ is any non-empty subset of $\{1, 2, \ldots, m\}$, representing which classes are to be admitted/served.  For $i= 0,$ the only admissible action is $\emptyset$. 
One can easily
see that $\mathcal{S}^+(i,a) = \{i-1, i+1\}$ for $0<i < n$, $\mathcal{S}^+(i,a) = \{i-1\}$ for $i=n$, and $\mathcal{S}^+(i,a) = \{ i+1\}$ for $i=0$, for an admissible action $a$.\footnote{In this example we allow the admissible action set to be state-dependent, for which the proof of the regret lower bound still holds  as studied in the discrete-time setting \cite{burnetas1997optimal}.}
}

The second assumption is on the rates of the holding times.

\begin{assumption} \label{assume-rate}
We assume there are two known constants $\lambda_{\min}, \lambda_{\max} \in (0,\infty)$, such that
\begin{align*}
\lambda_{\min} \le \lambda(s,a) \le \lambda_{\max}, \quad \text{for all $s \in \mathcal{S}, a \in \mathcal{A}$}.
\end{align*}
\end{assumption}

This assumption on bounded holding time rates is natural since the state and action spaces are both finite. \rev{In implementing our proposed learning algorithm, we need to specify {\it some} values of $\lambda_{\min}$ and $\lambda_{\max}$.} For practical applications and implementations, one may be able to have some rough estimates of bounds on the holding time rates based on knowledge about the problems at hand.

We denote by $\mathcal{H}$ the collection of CTMDPs that satisfy the above two assumptions.


 Given a CTMDP $\mathcal{M}$,  set
\begin{eqnarray} \label{eq:ergodic-control}
\rho^*(\mathcal{M})=\max_{\pi}  \left\{ \rho_{s_0}^{\pi}(\mathcal{M}) :=  \limsup_{T \rightarrow \infty} \frac{1}{T} \mathbb{E}_{s_0}^{\pi} \sum_{n=0}^{N(T) -1} r(s_n, a_n)  \right\},
\end{eqnarray}
where the maximum is over all policies under the assumption of full knowledge of $\mathcal{M}$'s parameters, and we recall that $N(T)$ denotes the total number of actions made up to time $T$. This quantity $\rho^*(\mathcal{M})$, or simply $\rho^*$,  is the (ground truth) optimal long-run average reward. For a CTMDP satisfying Assumptions~\ref{assume-transition} and \ref{assume-rate},
 there exist a function $h^*: \mathcal{S} \rightarrow \mathbb{R}$ and a constant $\rho^*$ that is independent of the initial state $s_0$ such that the following Bellman optimality equation holds for the average reward CTMDP \eqref{eq:ergodic-control}:
\begin{eqnarray} \label{eq:opt-eqn}
0  = \max_{a \in \mathcal{A}} \left\{r(s, a) - \rho^*/\lambda(s,a)+  \sum_{j \in \mathcal{S}} p(j|s, a) h^*(j) - h^*(s)  \right\},  \quad s \in \mathcal{S};
\end{eqnarray}
see, e.g., Chapter 11.5.3 of \cite{puterman2014markov}.
Here, $\rho^*$ is the optimal average reward per unit time given in \eqref{eq:ergodic-control}, and $h^*$ is called the bias function or the relative value function (defined up to an additive constant).
In addition, the maximizer of the right hand side of \eqref{eq:opt-eqn}, which is
a function of the state $s$, constitutes
a stationary, deterministic Markov policy that is optimal for problem \eqref{eq:ergodic-control}; see Theorem 11.4.6 and Chapter 11.5.3 of \cite{puterman2014markov}. We call this policy the optimal greedy policy. In the sequel, when we need to stress the dependence of $\rho^*$ and $h^*(s)$ on the underlying CTMDP $\mathcal{M}$, we will write them as $\rho^*(\mathcal{M})$ and $h^*(s; \mathcal{M})$ respectively.

 We measure the performance of a learning algorithm by the regret defined below.  
\begin{definition}\label{regret-def}
Given a learning algorithm $\mathcal{G}$ that generates action $a_n$ at state $s_n$,
 its regret  up to time $T$ is defined by
\begin{eqnarray} \label{eq:regret_T}
R_T^{\mathcal{G}} (s_0, \mathcal{M}) = T \rho^* - \sum_{n=0}^{N(T) -1} r(s_n, a_n),
\end{eqnarray}
where $\rho^*$ is given by \eqref{eq:ergodic-control} and $\sum_{n=0}^{N(T) -1} r(s_n, a_n)$ is the accumulated reward collected under $\mathcal{G}$.
\end{definition}

This notion of regret, studied by \cite{Jaksch2010} in the discrete-time setting, is defined against  the long-run average criterion $\rho^*$. As in \cite{Jaksch2010}, one can show that the gap between $T \rho^*$ and \rev{the optimal expected total reward of the true CTMDP (when it is known) with horizon length $T$ is bounded by $\max_{s, s' \in \mathcal{S} } \left( h^*(s) - h^*(s') \right) $, which is a finite constant independent of $T$. }

In the study of regret upper bounds, we will employ the well-known {\it uniformization} technique. Specifically, noting $\lambda(s,a)\le \lambda_{\max}$ for all $s,a$ under Assumption~\ref{assume-rate},
we rewrite equation~\eqref{eq:opt-eqn} as follows \cite[Chapter 11.5.3]{puterman2014markov}
\begin{eqnarray} \label{eq:uniform-eqn}
\check h^*(s)  = \max_{a \in \mathcal{A} } \left\{\check r(s, a)  -\check \rho^* +  \sum_{j \in \mathcal{S}} \check p(j|s, a) \check h^*(j)  \right\},
\end{eqnarray}
where ${ \check r(s, a) = r(s, a) \lambda(s,a)/\lambda_{\max} }$, $\check h^* = h^*$, $ \check p(j|s, a) = p(j|s, a) \lambda(s,a)/\lambda_{\max}$ for $j \ne s$, $ \check p(j|s, a) = 1- (1- p(j|s, a)) \lambda(s,a)/\lambda_{\max}$ for $j= s$ and $\check \rho^* = \rho^*/\lambda_{\max}$.  This is called the uniformization transformation. This transformation  leads to a new but equivalent  CTMDP where the decision times are determined by the arrival times of a homogeneous Poisson process with uniform rate $\lambda_{\max}$ that no longer depends on  the state. Moreover,  we recognize that  \eqref{eq:uniform-eqn} is the optimality equation for an average-reward {\it discrete-time} MDP with reward function $\check r$ and transition probabilities $(\check p(j|s, a))_{s,j,a}$. This implies that the original CTMDP model \eqref{eq:ergodic-control} is also equivalent to a (uniformized) discrete-time MDP model under the \textit{infinite-horizon} average-reward criterion.
 This being said, we can {\it not} directly apply the existing regret bound results for discrete-time MDPs to obtain regret bounds for the CTMDPs we consider. 
This is because we are ultimately interested in how the total regret of a learning algorithm grows as a function of time, whereas
on a \textit{finite horizon} the continuous-time and uniformized discrete-time MDPs are {\it not} equivalent -- indeed it is unclear how the gap between the two problems  scale as a function of time. This represents the major challenge we need to overcome in our regret analysis.

\section{Asymptotic \rev{Instance-Dependent} Regret Lower Bound} \label{sec:lowerbound}
We establish \rev{instance-dependent} regret lower bound for learning CTMDPs. 

\subsection{Preliminaries}\label{sec:LB-pre}
We first introduce some definitions that will be used later.

A learning algorithm $\mathcal{G}$ is called {\it uniformly fast convergent} (UF) if
 \begin{align}\label{eq:UF}
 \mathbb{E}_{s_0}[  R_T^{\mathcal{G}} (s_0, \mathcal{M}) ] = o(T^{\alpha})  \quad \text{as $T \rightarrow \infty$},
 \end{align}
 for any $\alpha>0, s_0 \in \mathcal{S}$ and any CTMDP $\mathcal{M} \in \mathcal{H}$, \rev{where for notational simplicity we use $\mathbb{E}_{s_0}$ to denote the expectation under the model $\mathcal{M}$, algorithm $\mathcal{G}$ and  starting state $s_0$}. Denote by $C_F$ the class of UF algorithms. {\it A priori} it is not clear whether or not the set $C_F$ is empty; however, we will develop an algorithm in Section~\ref{sec:upperbound} which shows that UF algorithms indeed exist in average-reward CTMDPs. On the other hand, a UF algorithm is defined  to perform uniformly well on {all} CTMDP instances, but may not perform arbitrarily well on a specific instance; so it is used to derive a lower bound (Theorem~\ref{thm:main}).

For a learning algorithm $\mathcal{G},$ define the $N-$th decision epoch $S_N = \sum_{i=0}^{N-2} \tau_i$ with the convention that $ \sum_{i=0}^{-1} \tau_i =0$, and define the expected regret up to $S_N$ by 
\begin{eqnarray} \label{eq:regret_SN}
\mathcal{R}^{\mathcal{G}}_N(s_0, \mathcal{M}) :=  \mathbb{E}_{s_0} [R_{S_N}^{\mathcal{G}} (s_0, \mathcal{M}) ]= \mathbb{E}_{s_0}[ \rho^* \cdot S_N - \sum_{n=0}^{N-1} r(s_n, a_n)],
\end{eqnarray}
where $R_T^{\mathcal{G}} (s_0, \mathcal{M})$ is defined in \eqref{eq:regret_T} for any $T>0.$

Let $O(s; \mathcal{M})$ be the set of maximizers of the right hand side of the equation~\eqref{eq:opt-eqn}. Define the ``suboptimality gap" of an action $a$ at state $s$ by
\begin{align} \label{eq:sub-gap}
\phi^*(s, a ; \mathcal{M}) = - r(s, a) + \rho^*/\lambda(s, a) -  \sum_{j \in \mathcal{S}}p (j|s, a) h^*(j) + h^*(s).
\end{align}
Clearly,  $\phi^*(s, a ; \mathcal{M}) =0$ if $a \in O(s; \mathcal{M})$, and
$\phi^*(s, a ; \mathcal{M})>0$ for any $s$ and $a \notin O(s; \mathcal{M})$.

\rev{In the following we introduce several terminologies and notations concerning critical state--action pairs that are natural extensions of their  discrete-time counterparts \cite{burnetas1997optimal}. They will be useful in deriving our instance-dependent regret lower bound.}
\begin{itemize}
\item For $(s, a)$ such that $a \notin O(s; \mathcal{M})$, let $\mathcal{\overline M} \in \mathcal{H} $ be the modification of the original model $\mathcal{M}$ in the following manner:
the transition probability distribution $p(\cdot|s,a)$ is modified to $\bar p(\cdot|s,a) \in \Theta(s, a)$, the rate of holding time $\lambda(s, a)$ is modified to $\bar \lambda(s, a) \in [\lambda_{\min}, \lambda_{\max}]$, and all the other parameters remain unchanged. When we need to stress the dependence of $\mathcal{\overline M}$ on $s, a,  \mathcal{ M}, \bar p(\cdot|s,a), \bar \lambda$, we will write it as $\mathcal{\overline M}(s, a; \mathcal{ M} , \bar p(\cdot|s,a), \bar \lambda(s, a))$.  

\item For $(s, a)$ with $a \notin O(s; \mathcal{M})$, denote by $\Delta \Theta(s, a ; \mathcal{M})$ the set of parameter values ($\bar p(\cdot|s,a),  \bar \lambda(s, a)$) that make action $a$ at state $s$ uniquely optimal under the modified model $\mathcal{\overline M}$. \rev{The uniqueness property of the optimal action is needed for technical reasons (see the proofs of  Lemma~\ref{lem: UF-N} and Proposition~\ref{lem: TN-LB}).} We also denote the set of critical state-action pairs  by $B(\mathcal{M}) = \{(s, a): a \notin O(s; \mathcal{M}), \Delta \Theta(s, a ; \mathcal{M}) \neq \emptyset \}$. \rev{Intuitively, a state-action pair $(s,a)$ is critical if $a$ is not optimal at state $s$ under the model $\mathcal{M}$, yet there exists a modified model $\mathcal{\overline M}$ under which $a$ becomes uniquely optimal  at $s$. }


\item  Define the `distance' (\rev{or difference}) between the two models $\mathcal{M}$ and $ \mathcal{ \overline M}$ at $(s, a)$:
\begin{align}\label{eq:KLdiff}
& \rev{KL_{\mathcal{M}, \mathcal{ \overline M}} (s,a) }    
 : = KL (p(\cdot|s,a), \bar p(\cdot|s,a) )  + KL(\tau(s, a), \bar \tau(s, a)) \\
&= \sum_{j \in \mathcal{S}} { p(j|s,a)} \log( p(j|s,a)/ \bar  p(j|s,a)) + \log(\lambda(s, a)/\bar \lambda(s, a)) + \bar \lambda(s, a) /\lambda(s, a) -1, \nonumber
\end{align}
where $\tau(s, a)$ and $\bar \tau(s, a)$ are the exponential distributions with parameters $\lambda(s, a)$ and $ \bar \lambda(s, a)$ respectively, and  $KL(\cdot, \cdot)$ is the KL divergence between two probability distributions. Note that $KL (p(\cdot|s,a), \bar p(\cdot|s,a) )$ is well defined since both $p(\cdot|s,a)$ and $\bar p(\cdot|s,a)$ are in the set $\Theta(s, a)$, and
 we use the fact that the KL divergence of two exponential distributions with parameters $\alpha$ and $\beta$  is  $\log(\alpha/\beta) + \beta/\alpha -1$.

\item Let
\rev{\begin{align} \label{eq:K}
K(s, a ; \mathcal{M})= \inf \{{KL_{\mathcal{M}, \mathcal{ \overline M}} (s,a) } : [ \bar p(\cdot| s, a), \bar \lambda(s, a)] \in \Delta \Theta(x, a ;\mathcal{M})  \},
\end{align}
which intuitively captures the minimum distance between the model $\mathcal{M}$ and the set of modified (or alternative) models $\mathcal{ \overline M}$.}
For fixed $(s, a)$, $K(s, a ; \mathcal{M})=0$ if $a \in O(s; \mathcal{M})$.
When $a \notin O(s; \mathcal{M})$, $0<K(s, a ; \mathcal{M})<\infty$
if $(s, a) \in B(\mathcal{M})$ and $K(s, a ; \mathcal{M})=\infty$ if $(s, a) \notin B(\mathcal{M})$. \rev{As we will see later in Proposition~\ref{lem: TN-LB}, for a critical state-action pair $(s,a)$, the number of visits to $(s,a)$ up to the $N-$th decision step under any UF algorithm is, asymptotically, at least $\frac{\log N}{K(s, a; \mathcal{M})}$. Hence, $K(s, a; \mathcal{M})$ is a measure of importance of the critical state-action pair $(s,a)$.}

\item Define
\begin{align} \label{eq:CM}
C(\mathcal{M})= \sum_{(s, a) \in B( \mathcal{M}) } \frac{\phi^*(s, a ; \mathcal{M})}{K(s, a ; \mathcal{M}) },
\end{align}
with the convention that the fraction is 0 if both the numerator and denominator are 0. This constant $C(\mathcal{M})$ provides an aggregate measure of importance of all the critical state-action pairs, which in turn characterizes the level of difficulty for  learning in the CTMDP $\mathcal{M}$ as we will see in Theorem~\ref{thm:main}. \rev{Intuitively, the suboptimality gap $\phi^*(s, a ; \mathcal{M})$ quantifies the regret incurred when one selects a sub-optimal action $a$ for state $s$, and $1/K(s, a; \mathcal{M})$ (multiplied by $\log N$) quantifies the minimal number of times of such an sub-optimal action has to be explored.}
Note that $C(\mathcal{M})=0$ only in the degenerate cases when  all policies are optimal (i.e. $a \in O(x;\mathcal{M})$ for all $a$) and/or none of the non-optimal actions can be made optimal by changing only its transition probability vector and the rate of holding times (i.e.
$ \Delta \Theta(s, a ; \mathcal{M}) = \emptyset$). 

\end{itemize}


\subsection{Main result}\label{sec:LB-main}
This section provides the main result on the asymptotic lower bound for the regret.

\begin{theorem}[Logarithmic instance-dependent regret lower bound]\label{thm:main}
For any learning algorithm $\mathcal{G}$ that is UF and any CTMDP $\mathcal{M} \in \mathcal{H}$, the expected regret up to $N-$th decision epoch satisfies
\begin{align} \label{eq:lower-bound-N}
\liminf_{N \rightarrow \infty} \frac{ \mathcal{R}_N^{\mathcal{G}}(s_0, \mathcal{M}) } { \log N}  \ge C(\mathcal{M}),
\end{align}
where the instance-dependent constant $C(\mathcal{M})$ is given in \eqref{eq:CM}. Moreover
\begin{align}\label{eq:coro-LB}
\liminf_{T \rightarrow \infty}  \frac{\mathbb{E}_{s_0}[R_T^{\mathcal{G}} (s_0, \mathcal{M})] } { \log T}  \ge C(\mathcal{M}).
\end{align}

\end{theorem}
A proof of Theorem~\ref{thm:main} is deferred to Section~\ref{sec:LB-analysis}.

\begin{remark}
In deriving  the instance-dependent lower bound, the assumption that the holding time is exponential is used only at two places: Proof of \eqref{eq:coro-LB} and Proof of Lemma~\ref{lem: UF-N} in Section~\ref{sec:LB-analysis}. 
 In both places, the exponential distribution facilitates the comparison of counting processes and the conversion of the regret bound in terms of discrete decision epoch $N$ to the regret bound of continuous time $T$.  Specifically, together with Assumption~\ref{assume-rate} (which involves only bounds on mean holding times), the exponential assumption  allows us to bound the counting process associated with the number of decision epochs under any learning algorithm by tractable Poisson processes. On the
 other hand, it is possible to extend \eqref{eq:coro-LB} to general semi-Markov decision processes which extend CTMDPs by allowing general holding time distributions. In such an extension, one needs extra assumptions on the holding time distributions so that the aforementioned comparison of counting processes still works. For instance, one may need bounds on the failure rate of the residual life time distribution of the counting process associated with decision epochs under any learning algorithm \cite{whitt1981comparing}. This in turn requires more information about the distribution of the holding times beyond the mere bounds on the mean. Here we do not pursue this extension in this paper.
\end{remark}

\subsection{Proof of Theorem~\ref{thm:main}} \label{sec:LB-analysis}
In this section we prove Theorem~\ref{thm:main}.
Given a learning algorithm $\mathcal{G},$ denote by $T_N(s, a)$ the number of occurrences of the state-action pair $(s, a)$ up to the $N-$th decision epoch, i.e. $T_N(s, a) = \sum_{n=0}^{N-1} 1_{s_n =s, a_n =a}$. 
We need several technical results for the proof of Theorem~\ref{thm:main}. For the ease of notation, unless specified otherwise, we use $\mathbb{P}_{s_0}$ and $\mathbb{E}_{s_0}$ to denote the probability and expectation under the model $\mathcal{M},$ algorithm $\mathcal{G}$ and initial state $s_0$ respectively.

\begin{proposition} \label{lem:decomp}
For any algorithm $\mathcal{G}$, the expected regret up to the $N-$th decision epoch satisfies,
\begin{align} \label{eq:regret-decomp}
\mathcal{R}_N^{\mathcal{G}}(s_0, \mathcal{M}) = \sum_{s \in \mathcal{S}} \sum_{a \notin O(s, \mathcal{M})}\mathbb{E}_{s_0}[T_N(s, a)] \cdot \phi^*(s, a ; \mathcal{M}) + O(1),  \quad \text{as $N \rightarrow \infty,$}
\end{align}
where $\phi^*(s, a ; \mathcal{M})$ is given in \eqref{eq:sub-gap}.
\end{proposition}

\begin{proof} 
We adapt proof of Proposition 1 in \cite{burnetas1997optimal} to the setting of CTMDPs.
Write
\begin{eqnarray*}
D_N^{\mathcal{G}}(s_0, \mathcal{M}) = \mathcal{R}_N^{\mathcal{G}}(s_0, \mathcal{M}) +h^*(s_0; \mathcal{M}) = \mathbb{E}_{s_0}[ \rho^* \cdot S_N - \sum_{n=0}^{N-1} r(s_n, a_n)] + h^*(s_0; \mathcal{M}),
\end{eqnarray*}
where $h$ is the bias function in the average reward optimality equation \eqref{eq:opt-eqn} and $S_N = \sum_{i=0}^{N-2} \tau_i$ with the convention that $ \sum_{i=0}^{-1} \tau_i =0$.
Then we can compute
\begin{align*}
D_N^{\mathcal{G}}(s_0,  \mathcal{M}) & = \mathbb{E}_{s_0}\left(   \mathbb{E}_{s_0}[  \rho^* \cdot S_N - \sum_{n=0}^{N-1} r(s_n, a_n)  |(a_0, s_1)] \right)  + h^*(s_0;  \mathcal{M}) \\
&=   \mathbb{E}_{s_0} \left(  \frac{\rho^*}{\lambda(s_0,  a_0) }  -  r(s_0, a_0)  + \mathbb{E}_{s_0}[  \rho^* \cdot \sum_{i=1}^{N-2} \tau_i - \sum_{n=1}^{N-1} r(s_n, a_n)  |(a_0, s_1)] \right)  \\
& \qquad  + h^*(s_0;  \mathcal{M}) \\
& =  \mathbb{E}_{s_0}[  \phi^*(s_0, a_0;  \mathcal{M})] +  \mathbb{E}_{s_0} [D_{N-1}^{\mathcal{G}}(s_1,  \mathcal{M})  ],
\end{align*}
where 
$\tau_i$ is the holding time at a state after the $(i+1)-$th decision epoch, and the last equality follows by adding and subtracting the same quantity
$ \mathbb{E}_{s_0}(  \mathbb{E}_{(a_0, s_1)} [h^*(s_1; \mathcal{M}) ] ) $ from the second last line of the above equation. It then follows from the above recursion that
\begin{align*}
D_N^{\mathcal{G}}(s_0, \mathcal{M}) = \sum_{n=0}^{N-1} \mathbb{E}_{s_0}[  \phi^*(s_n, a_n;\mathcal{M})]  = \sum_{s \in \mathcal{S}} \sum_{a \notin O(s,\mathcal{M})} \mathbb{E}_{s_0}[T_N(s, a)] \phi^*(s, a ; \mathcal{M}),
\end{align*}
where we use the fact that $\phi^*(s, a ; \mathcal{M})=0$ for $a \in O(s; \mathcal{M})$ in the last equality. Because the state space is finite, we know $\sup_{s_0 \in\mathcal{S}} |h^*(s_0; \mathcal{M})|$ is a finite constant which is independent of $N$.  Hence we obtain \eqref{eq:regret-decomp} and the proof is complete.
\end{proof}

\begin{proposition}\label{lem: TN-LB}
For any UF algorithm $\mathcal{G}$ and critical state-action pair $(s, a) \in B(\mathcal{M})$,
\begin{align*}
\liminf_{N \rightarrow \infty} \mathbb{E}_{s_0}[T_N(s, a)] / \log N \ge \frac{1} {K(s, a ; \mathcal{M})},
\end{align*}
where $K(s, a ; \mathcal{M})$ is given in \eqref{eq:K}.
\end{proposition}

The proof of Proposition~\ref{lem: TN-LB} is long and it requires several intermediate lemmas.
The essential  idea relies on a  change-of-measure argument similar to that in \cite{burnetas1997optimal}, although considerable difficulty arises in applying the argument to the continuous-time case due to the inherent random holding time. In the CTMDP setting, the key observation is the following.
The likelihood ratio of the two models $\mathcal{M}$ and $\mathcal{\overline M}$ given the history $\omega_N$ up to $N$-th decision epoch is given by 
\begin{align} \label{eq:LR1}
L(\mathcal{M} , \mathcal{\overline M}; \omega_N) = \mathbb{P}^{\mathcal{M}}_{s_0} (s_0, a_0, \tau_0, s_1, \ldots, \tau_{N-1}, s_N) /  \mathbb{P}^{ \mathcal{\overline M}}_{s_0} (s_0, a_0, \tau_0, s_1, \ldots, \tau_{N-1}, s_N).
\end{align}
 We can compute
\begin{align*}
&\mathbb{P}^{\mathcal{M}}_{s_0} (s_0, a_0,\tau_0, s_1, \ldots, \tau_{N-1}, s_N) \\
&= \prod_{k=0}^{N-1} \mathbb{P}^{\mathcal{M}}_{s_0} (a_k| s_0, a_0, \tau_0, s_1, \ldots, \tau_{k-1}, s_k) \cdot \mathbb{P}^{\mathcal{M}}_{s_0}( \tau_k |s_k, a_k) \mathbb{P}^{\mathcal{M}}_{s_0} ( s_{k+1} |s_k, a_k) .
\end{align*}
Since the algorithm $\mathcal{G}$ selects actions only using the history, the action $a_k$ is taken with the same probability given the history under $\mathcal{M}$ and $\mathcal{\overline M}.$
Hence, we have
\begin{align} \label{eq:LR2}
L(\mathcal{M} , \mathcal{\overline M}; \omega_N) & =  \prod_{k=0}^{N-1} \frac{ \mathbb{P}^{\mathcal{M}}_{s_0}  ( \tau_k |s_k, a_k) \cdot \mathbb{P}^{\mathcal{M}}_{s_0} ( s_{k+1} |s_k, a_k) }{ \mathbb{P}^{ \mathcal{\overline M}}_{s_0} ( \tau_k |s_k, a_k) \cdot \mathbb{P}^{ \mathcal{\overline M}}_{s_0} ( s_{k+1} |s_k, a_k) }.
\end{align}
Let the discrete random variable $Z_j(s, a) \in \mathcal{S}$ denote the state visited immediately after the $j-$th occurrence of $(s, a),$ and  $t_j(s, a)$ the holding time at state $s$ after the $j-$th occurrence of $(s, a).$
For $k \ge 0,$ $q, \bar q \in \Theta(s, a)$ and $\lambda , \bar \lambda \in [\lambda_{\min}, \lambda_{\max}],$ define
\begin{align} \label{eq:Lambda}
\Lambda_k ( [q, \lambda], [ \bar  q, \bar \lambda ] ) = \prod_{j=1}^k \frac{q_{Z_j(s, a) } }{\bar q_{Z_j(s, a) }} \cdot \frac{\lambda e^{-\lambda t_j(s, a)}}{\bar \lambda e^{-\bar \lambda t_j(s, a)} },
\end{align}
representing the likelihood ratio of the two models $\mathcal{M}$ and $\mathcal{\overline M}$ corresponding to the history of transitions out of state $s$ under action $a$.  Finally,  denote by $\mathbf{P}_{(p(\cdot|s,a), \lambda(s, a))}$ the probability measure generated by $(Z_j(s, a), t_j(s, a))_{j \ge 1}$, with $Z_j(s, a)$ following distribution $p(\cdot|s,a)$ and $t_j(s, a)$ following an exponential distribution with rate $\lambda(s, a).$

\begin{lemma} \label{lem:LR-SLLN}
We have
\begin{itemize}
\item [(i)] If $\mathcal{\overline M} = \mathcal{\overline M}(s, a; \mathcal{ M} , \bar p(\cdot|s,a), \bar \lambda(s, a)) $, then
\begin{align} \label{eq:LR3}
L(\mathcal{M} , \mathcal{\overline M}; \omega_N)= \Lambda_{ T_N(s, a)}([p(\cdot|s,a), \lambda(s, a)], [ \bar p(\cdot|s,a), \bar \lambda(s, a)] ),
\end{align}
 where
$ T_N(s, a)$ is the number of occurrences of the state-action pair $(s, a)$ up to the $N-$th decision epoch. 
\item [(ii)] Let $(b_n)$ be an increasing sequence of positive constants with $b_n \rightarrow \infty$ as $n \rightarrow \infty.$ Then, as $n \rightarrow \infty$, for any $c>0.$
\begin{align*}
& \mathbf{P}_{(p(\cdot|s,a), \lambda(s, a))} \biggl[ \max_{k \le \floor{b_n}} \log \Lambda_k / \floor{b_n}  > \\
&   \qquad \qquad \qquad \qquad (1+ c) \cdot I( [p(\cdot|s,a), \lambda(s, a)], [ \bar p(\cdot|s,a), \bar \lambda(s, a)])  \biggr] = o(1),
\end{align*}
\end{itemize}
\end{lemma}
\begin{proof}
(i) We have  $ \mathbb{P}^{\mathcal{M}}_{s_0} ( s_{k+1} |s_k, a_k) =   \mathbb{P}^{ \mathcal{\overline M}}_{s_0} ( s_{k+1} |s_k, a_k)  $ unless $(s_k, a_k) = (s, a)$, by the construction of $\mathcal{\overline M}$. 
Similarly, $ \mathbb{P}^{ \mathcal{ M}}_{s_0}( \tau_k |s_k, a_k) =  \mathbb{P}^{ \mathcal{\overline M}}_{s_0} ( \tau_k |s_k, a_k)$ unless $(s_k, a_k) = (s, a)$. Hence, all the terms except those corresponding to $(s_k, a_k) = (s, a)$ are cancelled in \eqref{eq:LR2}, leading to the desired result \eqref{eq:LR3}.

(ii) Since $Z_j(s, a), j =1 ,2 ,\ldots$ are i.i.d.,  and so are $t_j(s, a), j =1 ,2 ,\ldots$, the strong law of large numbers yields that as $k \rightarrow \infty,$
\begin{align*}
&\log \Lambda_k([p(\cdot|s,a), \lambda(s, a)], [ \bar p(\cdot|s,a), \bar \lambda(s, a)]) /k  \\
& \quad  \rightarrow I( [p(\cdot|s,a), \lambda(s, a)], [ \bar p(\cdot|s,a), \bar \lambda(s, a)]),\quad \mbox{a.s.} \ \mathbf{P}_{(p(\cdot|s,a), \lambda(s, a))}.
\end{align*}
  Consequently,  $\max_{k \le \floor{b_n}} \log \Lambda_k / \floor{b_n} \rightarrow I( [p(\cdot|s,a), \lambda(s, a)], [ \bar p(\cdot|s,a), \bar \lambda(s, a)])$ a.s. for any increasing sequence of $(b_n)$ with $b_n \rightarrow \infty.$ The result then follows.
\end{proof}


Given an algorithm $\mathcal{G},$
denote by $T_N(s)$ the total number of visits of the embedded jump chain $(s_n)$ of the CTMDP to the state $s$ up to the $N-$th decision epoch, i.e., $T_N(s)= \sum_{n=0}^{N-1} 1_{s_n=s}$. Under Assumption~\ref{assume-transition}, we have the following result.

\begin{lemma}[Proposition~2 of \cite{burnetas1997optimal}]\label{lem:irred}
There exists some $\beta>0$ such that for any algorithm $\mathcal{G}$ and $s \in \mathcal{S},$
\begin{align*}
\mathbb{P}^{\mathcal{M}}_{s_0}(T_N(s) < \beta N) = o(1) , \quad \text{as $N \rightarrow \infty$.}
\end{align*}
\end{lemma}

\begin{lemma}\label{lem: UF-N}
If $\mathcal{\overline M} = \mathcal{\overline M}(s, a; \mathcal{ M} , \bar p(\cdot|s,a), \bar \lambda(s, a)) $, i.e., action $a$ is uniquely optimal at state $s$ under the model $\mathcal{\overline M}$,  then
for any UF algorithm $\mathcal{G}$,
\begin{align} \label{eq:UF_Q}
\mathbb{E}_{s_0}^{ \mathcal{ \overline M} }[T_N(s) - T_N(s, a)] = \mathbb{E}_{s_0}^{ \mathcal{ \overline M} }[\sum_{a' \ne a }T_N(s, a')] = o(N^{\alpha}), \quad \text{for any $\alpha>0$,}
\end{align}
\rev{where $\mathbb{E}_{s_0}^{ \mathcal{ \overline M} }$ is the expectation under the model $\mathcal{\overline M}$, algorithm $\mathcal{G}$ and  starting state $s_0$.}
\end{lemma}

\begin{proof}
Since $\mathcal{G}$ is UF and $\mathcal{ \overline M} \in \mathcal{H}$, by definition we have $\mathbb{E}_{s_0}[ R_T^{\mathcal{G}} (s_0, \mathcal{ \overline M})] = o(T^{\alpha})$ as $T \rightarrow \infty$ for any $\alpha>0$. Then, for any $\epsilon>0$, there exists some $T_0$ such that $\mathbb{E}_{s_0}[ R_T^{\mathcal{G}} (s_0, \mathcal{ \overline M})]\le \epsilon T^{\alpha}$ for $T \ge T_0$. It follows that $\mathbb{E}_{s_0}[ R_T^{\mathcal{G}} (s_0, \mathcal{ \overline M})]\le \epsilon T^{\alpha} + T_0 \rho^*$ for any $T \ge 0.$
By virtue of  \eqref{eq:regret_SN}, we get
\begin{align} \label{eq:RN-bound}
\mathcal{R}_N^{\mathcal{G}}(s_0, \mathcal{M}) :=  \mathbb{E}_{s_0} [R_{S_N}^{\mathcal{G}} (s_0, \mathcal{M}) ] =  \mathbb{E}_{s_0} \left( \mathbb{E}_{s_0} [R_{S_N}^{\mathcal{G}} (s_0, \mathcal{M}) | S_N ]  \right)
\le \epsilon  \cdot \mathbb{E}_{s_0}[S_N^{\alpha}] + T_0 \rho^*.
\end{align}

Denote by $(\tau_i)$ the sequence of holding times under any algorithm $\mathcal{G}$ applied to the CTMDP $\mathcal{ \overline M} \in \mathcal{H}$. Since $S_N = \sum_{i=0}^{N-2} \tau_i$, by Assumption~\ref{assume-rate}, it is possible to construct a probability space
and a sequence of i.i.d. exponential random variables $(\hat \tau_i)$ with rate $\lambda_{\min}$ such that the law of $(S_N: n \ge 1)$ is the same as the law of $(\tilde S_N: N \ge1)$ and $\tilde S_N \le \hat S_N = \sum_{i=0}^{N-2} \hat \tau_i$ for all $N$ and all sample paths; see, e.g., Section 2 of \cite{whitt1981comparing}.
 It follows that $\mathbb{E}_{s_0}[S_N^{\alpha}] = \mathbb{E}_{s_0}[\tilde S_N^{\alpha}] \le \mathbb{E}[ \hat S_N^{\alpha}] $. Without loss of generality we consider $\alpha \in (0,1)$. From the strong law of large numbers we have $\hat S_N^{\alpha}/N^{\alpha} \rightarrow (1/\lambda_{\min})^{\alpha}$ a.s. when $N \rightarrow \infty.$ Moreover, $\mathbb{E}[ \hat S_N^{\alpha}] / N^{\alpha} \le  \mathbb{E}[ \hat S_N] / N =  1/\lambda_{\min} \cdot (N-1)/N \le 1/\lambda_{\min} < \infty$ for each $N$. By the dominated convergence theorem we then obtain
$\lim_{ N \rightarrow \infty } \frac{ \mathbb{E}[ \hat S_N^{\alpha}]}{ N^{\alpha} } =  (1/\lambda_{\min})^{\alpha}.$
As a consequence,
$\limsup_{ N \rightarrow \infty } \frac{ \mathbb{E}_{s_0}[S_N^{\alpha}] }{N^{\alpha} } \le  (1/\lambda_{\min})^{\alpha}.$
It then follows from \eqref{eq:RN-bound} that $
 \limsup_{ N \rightarrow \infty } \frac{\mathcal{R}_{N}^{\mathcal{G}}(s_0, \mathcal{ \overline M}) } { N^{\alpha} }\le \epsilon  \cdot (1/\lambda_{\min})^{\alpha}.
$
Since $\epsilon>0$ is arbitrary, we obtain $\mathcal{R}_{N}^{\mathcal{G}}(s_0,  \mathcal{ \overline M} )= o(N^{\alpha})$.

From Proposition~\ref{lem:decomp}, we know that
for any learning algorithm $\mathcal{G}$, the expected regret up to $N-$th decision epoch satisfies
\begin{align*}
\mathcal{R}_N^{\mathcal{G}}(s_0, \mathcal{ \overline M}) = \sum_{x \in \mathcal{S}}  \sum_{a' \ne a } \mathbb{E}_{s_0}^{ \mathcal{ \overline M} }[T_N(x, a')]  \phi^*(x, a'; \mathcal{ \overline M}) + O(1),
\end{align*}
where $\{a\} = O(x; \mathcal{ \overline M})$.
Since $\mathcal{R}_{N}^{\mathcal{G}}(s_0, \mathcal{ \overline M})= o(N^{\alpha})$, and $\phi^*(x, a'; \mathcal{ \overline M})$ are all positive and uniformly bounded away from zero for $a' \ne a$, we arrive at \eqref{eq:UF_Q}.
\end{proof}

Now we are ready to prove Proposition~\ref{lem: TN-LB}.

\begin{proof}[Proof of Proposition~\ref{lem: TN-LB}]
Given all the previous lemmas, the proof is similar to that of Theorem 2 in \cite{burnetas1997optimal}; so we only outline the key steps while highlighting the differences. It suffices to show for any $\epsilon>0,$
\begin{align*}
\lim_{N \rightarrow \infty} \mathbb{P}_{s_0} \left( T_N(s, a) < \frac{(1-\epsilon) \log N}{ K(s, a ; \mathcal{M}) }\right)=0 .
\end{align*}
By Lemma~\ref{lem:irred}, it further remains to show
\begin{align*}
 \mathbb{P}_{s_0} \left( T_N(s, a) < \frac{(1-\epsilon) \log N}{ K(s, a ; \mathcal{M}) }, \quad T_N(s) \ge \beta N \right)=o(1), \quad \text{as $N \rightarrow \infty.$}
\end{align*}

Let $\mathcal{G}$ be a UF algorithm, $(s, a) \in B(\mathcal{M})$ and $\delta = \epsilon/(2- \epsilon)>0.$ By the definition of
$K(s, a ; \mathcal{M})$ in \eqref{eq:K}, one can readily verify that
\[ (1-\epsilon)/K(s, a; \mathcal{M}) \le (1-\delta)/I( [p(\cdot|s,a), \lambda(s, a)], [ \bar p(\cdot|s,a), \bar \lambda(s, a)]). \]
Write $b_N = (1-\delta) \log N / I( [p(\cdot|s,a), \lambda(s, a)], [ \bar p(\cdot|s,a), \bar \lambda(s, a)]) $, $\omega_N$ the history up to decision epoch $N$, and $E_N^{\delta} = \{\omega_N: T_N(s) \ge \beta N,  \quad T_N(s, a) < b_N \}$. Then it suffices to prove
\begin{align} \label{eq:PAN}
\mathbb{P}_{s_0}(E_N^{\delta})= o(1), \quad \text{as $N \rightarrow \infty.$}
\end{align}

Since $\mathcal{G}$ is UF, and action $a$ is now uniquely optimal under $\mathcal{ \overline M} = \mathcal{\overline M}(s, a; \mathcal{ M} , \bar p(\cdot|s,a), \bar \lambda(s, a))$, it follows from Lemma~\ref{lem: UF-N} that
$
\mathbb{E}_{s_0}^{ \mathcal{ \overline M} }[T_N(s) - T_N(s, a)] = \mathbb{E}_{s_0}^{ \mathcal{ \overline M} }[\sum_{a' \ne a }T_N(s, a')] = o(N^{\alpha}).
$
Thus the Markov inequality implies
\begin{align}\label{eq:PM}
 \mathbb{P}_{s_0}^{\mathcal{G}, \mathcal{ \overline M} }( E_N^{\delta}  )  \le  \mathbb{P}_{s_0}^{\mathcal{G}, \mathcal{ \overline M} } ( T_N(s) - T_N(s, a)  \ge \beta N -b_N)  = \frac{o(N^{\delta/2})}{\beta N - b_N} = o(N^{\delta/2 -1}).
\end{align}
Let $ B_N^{\delta} = \{\omega_N: \log \Lambda_{T_N(s, a)}([p(\cdot|s,a), \lambda(s, a)], [ \bar p(\cdot|s,a), \bar \lambda(s, a)] ) \le (1-\delta/2) \log N\} $. 
To show \eqref{eq:PAN}, we first show $ \mathbb{P}_{s_0}( E_N^{\delta}   \cap B_N^{\delta}) = o(1).$ Using the likelihood ratio \eqref{eq:LR1}, equation~\eqref{eq:LR3} and the definition of $B_N^{\delta}$, we have
\begin{align*}
 \mathbb{P}_{s_0}( E_N^{\delta}   \cap B_N^{\delta})  \le  e^{(1- \delta/2)  \log N} \cdot  \mathbb{P}_{s_0}^{\mathcal{G}, \mathcal{ \overline M} } (E_N^{\delta} \cap   B_N^{\delta}) \le  N^{1- \delta/2} \cdot \mathbb{P}_{s_0}^{\mathcal{G}, \mathcal{ \overline M} }( E_N^{\delta}  ) =  o(1),
\end{align*}
where the last equality is due to \eqref{eq:PM}.

We next show $ \mathbb{P}_{s_0}( E_N^{\delta} \cap \bar B_N^{\delta}) = o(1).$ Note that
\begin{align*}
&\mathbb{P}_{s_0}( E_N^{\delta} \cap  \bar B_N^{\delta}) \\
& \le  \mathbf{P}_{(p(\cdot|s,a), \lambda(s, a))} \left( \max_{k \le \floor{b_N}} {\log \Lambda_k} /\floor{b_N} >\frac{1- \delta/2}{1- \delta}  \cdot I( [p(\cdot|s,a), \lambda(s, a)], [ \bar p(\cdot|s,a), \bar \lambda(s, a)]) \right) .
\end{align*}
 From part (ii) of Lemma~\ref{lem:LR-SLLN}
we obtain that $ \mathbb{P}_{s_0}( E_N^{\delta}  \cap \bar B_N^{\delta}) = o(1).$ Hence we have proved \eqref{eq:PAN}. The proof is therefore complete.
\end{proof}

\begin{proof}[Proof of Theorem~\ref{thm:main}]
Equation \eqref{eq:lower-bound-N} directly follows from Propositions~\ref{lem:decomp} and~\ref{lem: TN-LB} along with the definition of $C(\mathcal{M})$ in \eqref{eq:CM}.

We next prove \eqref{eq:coro-LB}.
Recall that
$
R_T^{\mathcal{G}}(s_0, \mathcal{ M})  = \rho^* \cdot T -  \sum_{n=0}^{N(T)-1} r(s_n, a_n),
$
where $N(T)$ is the number of decision epochs by time $T$ under the learning algorithm $\mathcal{G}$. Since the reward is bounded by one, we have
\begin{align*}
| R_T^{\mathcal{G}}(s_0, \mathcal{ M})  -  R_{S_{N(T) +1}}^{\mathcal{G}}(s_0, \mathcal{  M}) | \le \rho^* \cdot [S_{N(T) +1} - S_{N(T)}] + 1,
\end{align*}
where we recall $S_n$ denotes the $n-$th decision epoch under a given algorithm $\mathcal{G}$.
It then follows from Assumption~\ref{assume-rate} that
\begin{align} \label{eq:TNT}
 \mathbb{E}_{s_0} | R_T^{\mathcal{G}}(s_0, \mathcal{ M})  - R_{S_{N(T)+1}}^{\mathcal{G}}(s_0, \mathcal{ M})   | \le  \frac{\rho^*}{\lambda_{\min}} + 1.
\end{align}
Dividing both sides by $\log T$ and sending $T$ to infinity, we have
$
\liminf_{T \rightarrow \infty}  \frac{\mathbb{E}_{s_0}[R_T^{\mathcal{G}} (s_0, \mathcal{M})] } { \log T}  \ge \liminf_{T \rightarrow \infty}  \frac{\mathbb{E}_{s_0}[ R_{S_{N(T) +1}}^{\mathcal{G}}(s_0, \mathcal{ M}) ]  }{ \log T}.
$
Hence, to prove \eqref{eq:coro-LB}  it suffices to show
\begin{align*}
\liminf_{T \rightarrow \infty} \mathbb{E}_{s_0}[ R_{S_{N(T) +1}}^{\mathcal{G}}(s_0, \mathcal{ M}) ]   / \log T  \ge C(\mathcal{M}).
\end{align*}
From \eqref{eq:lower-bound-N}, we know that for any $\epsilon>0$, there exists $N_0>0$ such that $\mathcal{R}_N^{\mathcal{G}}(s_0, \mathcal{  M})  / \log N  \ge ( C(\mathcal{M}) - \epsilon)$ when $N \ge N_0.$ Hence for all $N \ge 1,$
\begin{align*}
 \mathcal{R}_N^{\mathcal{G}}(s_0, \mathcal{ M}) = \mathbb{E}_{s_0}[ R_{S_{N}}^{\mathcal{G}}(s_0, \mathcal{ M}) ] \ge (C(\mathcal{M})- \epsilon) \log N  - C_0,
\end{align*}
where $C_0:=\frac{\rho^* N_0}{\lambda_{\min}}$ is a constant that bounds the expected regret in the first $N_0$ decision steps under any algorithm.
It follows that 
\begin{align} \label{eq:LB-intermediate}
\mathbb{E}_{s_0}[ R_{S_{N(T) +1}}^{\mathcal{G}}(s_0, \mathcal{ M}) ] & = \mathbb{E}_{s_0} \left( \mathbb{E}_{s_0}[ R_{S_{N(T) +1}}^{\mathcal{G}}(s_0, \mathcal{ M}) | N(T)+1] \right) \nonumber \\
& \ge (C(\mathcal{M})- \epsilon) \mathbb{E}_{s_0}[ \log  (N(T) +1)] - C_0.
\end{align}
Note that due to Assumption~\ref{assume-rate},  we have under any algorithm $\mathcal{G},$ the counting process $\{N(t) -1: t \ge 0\}$ is sandwiched pathwise between two Poisson processes $A_1(\cdot)$ and $A_2(\cdot)$ with rates $\lambda_{\min}$ and $\lambda_{\max}$ respectively, i.e. $A_1(t) \le N(t) -1 \le A_2(t)$ for all $t \ge 0$ and all sample paths, see, e.g., Section 2 of \cite{whitt1981comparing}. Note that we consider $N(t) -1$ because there is a decision made at time 0 so that $N(0)-1=0$.
For a Poisson process $A(\cdot)$ with any rate $\mu>0$, we can infer from the strong law of large numbers that
\begin{align*}
 \lim_{T \rightarrow \infty}\mathbb{E} \left[ \log (A(T)+2) - \log(T)  \right] =\lim_{T \rightarrow \infty}\mathbb{E} \left[ \log \frac{A(T) +2}{T} \right] = \log(\mu),
\end{align*}
where in the last equality follows from the generalized dominated convergence theorem given that $\log \frac{A(T) +2}{T} \le \frac{A(T) +2}{T}$, and $\frac{A(T) +2}{T} \rightarrow \mu$ almost surely and in expectation as $T \rightarrow \infty$. Hence we have  $\lim_{T \rightarrow \infty}\mathbb{E}_{s_0}[ \log (A(T) +2 ) ]  / \log T  =1.$
Consequently
$\liminf_{T \rightarrow \infty} \frac{\mathbb{E}_{s_0}[ \log (N(T)+1)] } {\log T}  =1,$
which, together with \eqref{eq:LB-intermediate}, yields
\begin{align*}
 \liminf_{T \rightarrow \infty} \mathbb{E}_{s_0}[ R_{S_{N(T) +1}}^{\mathcal{G}}(s_0, \mathcal{ M}) ] / \log T \ge C(\mathcal{M})- \epsilon.
\end{align*}
The result then follows as $\epsilon$ is arbitrarily small.
\end{proof}

\section{The CT--UCRL Algorithm and Its Instance-Dependent Regret Upper Bound}\label{sec:upperbound}

 In this section we present the CT-UCRL algorithm for learning in CTMDP and establish a finite-time instance-dependent upper bound for its regret.


First, we introduce the following definition from  \cite{Jaksch2010} and \cite{fruit2017exploration}, which is a real-valued measure of the connectedness of an
MDP.

\begin{definition}
The diameter $D(\mathcal{M})$ of an MDP $\mathcal{M}$, either continuous time or discrete time,  is defined by
\begin{align} \label{eq:diameter}
D(\mathcal{M})= \max_{s, s' \in \mathcal{S}} \left\{  \min_{\pi: \mathcal{S} \rightarrow \mathcal{A}} \mathbb{E}^{\pi}[T(s') |s_0=s] \right\}
\end{align}
where $T(s')$ is the first time when state $s'$ is reached.
\end{definition}
\begin{remark}\label{remark:diam}
It is known that the diameter of a discrete-time MDP with finite state space is finite if and only if the MDP is communicating (i.e. for any pair of states $s, s'$, there exists a deterministic stationary policy under which the  probability of eventually reaching $s'$ starting from $s$ is positive); see \cite{Jaksch2010}.
In our case, the diameter $D(\mathcal{M})$ of the CTMDP $\mathcal{M}$ is also finite. This is because by Assumption~\ref{assume-rate} we can uniformize the CTMDP to obtain an equivalent discrete-time MDP $\mathcal{M}^{eq}$. It follows from the proof of Lemma 6 in \cite{fruit2017exploration} that $D(\mathcal{M}) = D(\mathcal{M}^{eq})/\lambda_{\max}$. By Assumption~\ref{assume-transition}, $\mathcal{M}^{eq}$ is communicating;  hence $D(\mathcal{M}^{eq})$ is finite.
\end{remark}



\subsection{The CT-UCRL algorithm} \label{sec:UCRL}
In this section we present the CT-UCRL algorithm. 


\subsubsection{Refined estimator for mean holding time} \label{sec:truncate-mean-HT}
In the CT-UCRL algorithm we need to estimate the unknown quantities, i.e., transition probabilities and mean holding times, with tight confidence bounds. For the former, one can simply use the empirical transition probabilities as in \cite{Jaksch2010}. However, in order to  obtain logarithmic regret bounds, the estimation of the mean holding times requires extra delicate analysis because the holding time follows an exponential distribution which has tails heavier  than those of a Gaussian distribution. This represents one of the main differences and difficulties in treating the continuous-time case.

Specifically, for a given state-action pair $(s,a)$, suppose $(X_i)_{i=1}^n$ are i.i.d exponential holding times with mean $1/\lambda(s,a)$. In the following we omit the dependancy of $\lambda$ on $(s,a)$ for notational simplicity.
Let $\bar X = \frac{1}{n} \sum_{i=1}^n X_i$ be the sample mean. Then it is known that the upper tail of the sample mean $\bar X$ is exponential instead of Gaussian, though the lower tail is Gaussian. Mathematically,
 we have for any $x>0$, 
\begin{align*}
\mathbb{P} \left( \lambda \bar X -1 \ge \frac{x}{n} + \sqrt{\frac{2x}{n}} \right) \le  e^{-x} , \quad \text{and} \quad \mathbb{P} \left( \lambda \bar X -1 \le - \sqrt{\frac{2x}{n}} \right) \le  e^{-x} .
\end{align*}



However, to establish the logarithmic regret bound of the CT-UCRL algorithm, it turns out in our analysis we need an estimator $\hat X$ (based on $n$ i.i.d observations) for the mean holding time $1/\lambda$ with two-sided Gaussian tail bounds.
Hence, we need to replace the empirical mean by other refined estimators. One simple choice (see e.g. \cite{bubeck2013bandits}) is the truncated empirical mean $\hat X$ defined by
\begin{align} \label{eq:truncate-mean}
\hat X = \frac{1}{n} \sum_{i=1}^n X_i \cdot 1_{X_i \le \sqrt{ \frac{2 i }{ \lambda_{\min}^2 \log (1/\delta)}}} ,
\end{align}
where the constant $2/\lambda_{\min}^2$ acts as an upper bound of the second moment $\mathbb{E}[X_i^2] = \frac{2}{ \lambda^2}$ by Assumption~\ref{assume-rate}. 
Lemma 1 of \cite{bubeck2013bandits} shows that for $\delta \in (0,1)$, the following estimate holds with probability at least $1- \delta$:
\begin{align} \label{eq:b1}
\left|\hat X - \frac{1}{\lambda}  \right| \le \frac{4  }{\lambda_{\min}} \sqrt{ \frac{ 2 \log(1/\delta)}{n} }.
\end{align}
Our proof of the logarithmic regret of the CT-CURL algorithm relies on  \eqref{eq:b1}.

\subsubsection{The CT-UCRL algorithm} \label{sec:CTUCRL}

The CT-UCRL learning algorithm in CTMDPs is presented as Algorithm~\ref{alg: CTUCRL}. Let us provide an overview of the algorithm and then elaborate some  details.

  \begin{algorithm}[!ht]
        \caption{The CT-UCRL Algorithm}
         \setstretch{0.95}
        \label{alg: CTUCRL}
        \begin{algorithmic}[1]
            \REQUIRE Confidence parameter $\delta \in (0,1), \lambda_{\min}, \lambda_{\max} \in (0, \infty), \mathcal{S}, \mathcal{A}$ and reward function $r$
            \vspace{1mm}
            \STATE Initialization: set $n=1$ and observe initial state $s_1.$
            \vspace{1mm}
            \FOR {episode $k=1,2,3,\dots$}\label{step:episode}  \vspace{1mm}
            \STATE 
            \textbf{Initialize episode $k$}:
            \\
           \begin{itemize}
\item [(a)] Set the start decision epoch of episode $k,$ $t_k := n$

\item [(b)] For all $(s,a)$ initialize the state-action counter for episode $k$, $v_k(s,a)=0.$ Further set the state-action counts prior to episode $k$ as $N_{k}(s,a):= \#\{i < t_k : s_i=s, a_i = a \}$.

\item [(c)] For $s, s' \in \mathcal{S}$ and $a \in \mathcal{A}$, set the observed accumulated transition counts piror to episode $k$ as $N_k(s,a, s'):=\#\{i < t_k : s_i=s, a_i = a , s_{i+1} = s'\} $.  Compute the empirical transition probabilities:
$\hat p(s'|s, a) = \frac{N_k(s,a, s')}{\max\{1,N_{k}(s,a) \}}.$
Also compute the \textit{truncated} empirical mean $\frac{1}{\hat \lambda(s,a)}$ for the holding time by (following \eqref{eq:truncate-mean}):
\begin{align*}
\frac{1}{\hat \lambda(s,a)} = \frac{1}{N_k(s,a)} \sum_{j=1}^{t_k-1} \tau_j \cdot 1_{s_j=s, a_j=a} \cdot 1_{\tau_j \le \sqrt{ \frac{2 T_j(s,a) }{ \lambda_{\min}^2 \log (1/\delta)}}},
\end{align*}
where $T_{j}(s,a):= \#\{i \le j: s_i=s, a_i = a\}$ denotes the state-action counter up to decision epoch $j.$
\end{itemize}

\STATE  \textbf{Compute policy $\tilde \pi_k:$}\\
         \begin{itemize}
\item [(a)] Let $\mathcal{C}_k$ be the set of all CTMDPs with states and actions as $\mathcal{M}$, and with transition probabilities $ \tilde p(\cdot |s, a)$ and holding time rate parameter $\tilde \lambda(s,a)$ such that for all $(s, a) \in \mathcal{S} \times \mathcal{A}$,
\begin{align}
&||\tilde p(\cdot |s, a) -  \hat p(\cdot |s, a)  ||_1 \le \sqrt{\frac{14S \log(2At_k/\delta)}{ \max\{1, N_k(s,a)\}} } \label{eq:P-CI}\\
 &   \left|\frac{1}{\tilde \lambda(s,a)} - \frac{1}{\hat \lambda(s,a)} \right| \le \frac{4 }{\lambda_{\min}} \sqrt{ \frac{14 \log(2 AS t_k/\delta)}{ \max\{1, N_k(s,a)\}}} \quad \text{and $\tilde \lambda(s,a) \in [\lambda_{\min}, \lambda_{\max}]$}. \label{eq:rate-CI}
\end{align}

\item [(b)] Find a CTMDP $\widetilde M_k \in \mathcal{C}_k$ and a policy $\tilde \pi_k$ such that 
\begin{align}
\rho_k^* :=  \min_{ s \in \mathcal{S}} \rho^{\tilde \pi_k}_s ( \widetilde M_k)  \ge  \max_{s \in \mathcal{S}}  \max_{\pi} \max _{\mathcal{M'} \in \mathcal{C}_k} \rho^{\pi}_s( \mathcal{M'})  - \frac{1}{\sqrt{t_k}}. \label{eq:optimistic-model}
\end{align}
\end{itemize}

 \STATE  \textbf{Execute policy $\tilde \pi_k:$}  \vspace{1mm}
\STATE  \textbf{while} $v_k(s_n, \tilde \pi_k(s_n)) < \max\{1, N_k(s_n, \tilde \pi_k(s_n))\}$  \textbf{do}
\begin{itemize}
\item [(a)] Choose the action $a_n= \tilde \pi_k(s_n)$, observe the holding time $\tau_n$ and the next state $s_{n+1}$
\item [(b)] Update $v_k(s_n, a_n) = v_k(s_n, a_n)  +1,$ and set $n=n+1.$ Here $n$ is decision epoch.
\end{itemize}

        \ENDFOR
        \end{algorithmic}
    \end{algorithm}

The algorithm proceeds in episodes $k=1, 2, 3, \ldots$ with variable lengths. At the start of episode $k$, the algorithm constructs estimators and confidence intervals for the unknown transition probabilities and mean holding times; see \eqref{eq:P-CI}--\eqref{eq:rate-CI}. These lead to a set of statistically plausible CTMDPs denoted by $\mathcal{C}_k$. The algorithm finds a CTMDP $\widetilde M_k \in \mathcal{C}_k$, referred to as the optimistic CTMDP, that maximizes the average reward among the plausible CTMDPs in certain sense (to be specified shortly)  and computes the corresponding optimal greedy policy $\tilde \pi_k$; see \eqref{eq:optimistic-model}.  The details of this computation will be discussed later. The policy $\tilde \pi_k$ is executed until the end of episode $k$ when the number of visits to some state-action pair doubles (Step 6 of Algorithm~\ref{alg: CTUCRL}). More specifically, episode $k$ ends when for some state-action pair $(s,a),$  the number of visits to $(s,a)$ in episode $k$, $v_k(s,a)$, equals $\max\{1, N_k(s,a)\}$ which is the total number of visits to $(s,a)$ up to and including episode $k-1$.
Then a new episode starts and the whole process repeats.

We now discuss \eqref{eq:optimistic-model}, the problem of  finding the CTMDP model in the confidence region with the nearly largest average reward. This problem can be solved by introducing an ``extended" CTMDP that combines all the CTMDPs in the confidence region and finding the average-reward optimal policy on the extended CTMDP. The solution  procedure, known as the extended value iteration, is discussed in \cite{fruit2017exploration} for SMDPs, though their confidence region differs from ours. Here we outline the key ideas for readers' convenience.

Introduce the extended CTMDP $\widetilde M_k^+$ with the state space $\mathcal{S}$ and a (state-dependent) action space 
$\mathcal{A}_s^+$ defined as
\rev{\begin{align*}
\mathcal{A}_s^+ = \{(a, \tilde p (\cdot |s, a) , \tilde \lambda(s,a)): a \in \mathcal{A}, (\tilde p (\cdot |s, a) , \tilde \lambda(s,a)) \in \mathcal{P}_k (s,a) \times C_k(s,a) \},
\end{align*} }
where $\mathcal{P}_k(s,a)$ and $C_k(s,a)$ are the confidence sets in \eqref{eq:P-CI} and \eqref{eq:rate-CI} respectively for a given state-action pair $(s,a)$. \rev{ Specifically,
\begin{align*}
& \mathcal{P}_k(s,a) :=\left\{ \tilde p \in \mathcal{P}(\mathcal{S}):  ||\tilde p -  \hat p(\cdot |s, a)  ||_1 \le \sqrt{\frac{14S \log(2At_k/\delta)}{ \max\{1, N_k(s,a)\}} }  \right\},\\
 & C_k(s,a):=  \left\{  \tilde \lambda \in [\lambda_{\min}, \lambda_{\max}]:  \left|\frac{1}{\tilde \lambda} - \frac{1}{\hat \lambda(s,a)} \right| \le \frac{4 }{\lambda_{\min}} \sqrt{ \frac{14 \log(2 AS t_k/\delta)}{ \max\{1, N_k(s,a)\}}}  \right\},
\end{align*}
where $\mathcal{P}(\mathcal{S})$ is the set of probability vectors on $\mathcal{S}$, $t_k$ denotes the start decision epoch of episode $k$, and $\hat p(\cdot |s, a) $ and $ \hat \lambda(s,a)$ denote the estimators for transition probabilities and holding time rates respectively.
The reward function, transition probabilities  and rates of holding times associated $\mathcal{A}_s^+$ are specified as follows.
At a given  state--action pair
$(s,a_s^+)\equiv (s,a, \tilde p (\cdot |s, a) , \tilde \lambda(s,a)) \in
\mathcal{S}\times \mathcal{A}_s^+, $
the reward  is $r^+(s, a_s^+) = r(s,a),$ the transition probabilities are $p^+( \cdot|s, a_s^+) = \tilde p (\cdot |s, a)$, and the rates of holding times are $\lambda^+(s, a_s^+) = \tilde \lambda(s,a)$. }
Then,
following a similar discussion as in Section 3.1.1 of \cite{Jaksch2010}, we conclude  that
finding a CTMDP $ \widetilde M  \in \mathcal{C}_k$ and a policy $\tilde \pi$ on $ \widetilde M$ such that $\rho^{\tilde \pi}_s ( \widetilde M) =  \max_{s, \pi,  \mathcal{M'} \in \mathcal{C}_k}   \rho^{\pi}_s( \mathcal{M'})$
for all initial state $s$ corresponds to finding  the average-reward optimal policy on $\widetilde M_k^+$.
Specifically, for any given CTMDP $\mathcal{M'} \in \mathcal{C}_k$ and any policy $ \pi' : \mathcal{S} \rightarrow \mathcal{A}$ on $\mathcal{M'}$, there is a policy $\pi^+$ on the extended CTMDP $\widetilde M_k^+$ such that the same transition probabilities, holding-time rates and rewards are induced by $ \pi'$ on $\mathcal{M'}$ and $\pi^+$ on $\widetilde M_k^+$. On the other hand, for each policy $\pi^+$ on $\widetilde M_k^+$, there is a CTMDP $\mathcal{M'} \in \mathcal{C}_k$ and a policy $ \pi'$ on it so that again the same transition probabilities, holding-time rates and rewards are induced.

It remains to discuss how to solve the extended CTMDP  $\widetilde M_k^+$ under the average-reward criteria for each episode $k$. One can apply the uniformization procedure, convert the extended CTMDP to an equivalent discrete-time MDP model $\widetilde M_{k, eq}^+$, and apply a value iteration scheme to solve $\widetilde M_{k, eq}^+$ \textit{approximately}.
Specifically, denote the state values of iteration $i$ by $u_i(s)$ for $s \in \mathcal{S}$. Then the value iteration on $\widetilde M_{k, eq}^+$ becomes the following \rev{
\begin{align} \label{eq:EVI}
& u_{i+1}(s) \\
 &= \max_{a_s^+ \in \mathcal{A}_s^+ } \left\{\check r(s, a_s^+)  +  \sum_{j \in \mathcal{S}} \check p(j|s, a_s^+) u_{i} (j)  \right\} \nonumber \\
& = \max_{a \in \mathcal{A} } \left\{ \max_{ \tilde \lambda(s, a) \in C_k(s,a) } \left\{  \left[ r(s, a) + \max_{ \tilde p(\cdot |s, a) \in \mathcal{P}_k(s,a) } \sum_{j \in \mathcal{S}} \tilde p(j|s, a) u_i (j) - u_i(s) \right] \frac{ \tilde \lambda(s,a)} {\lambda_{\max} } \right\} \right\} \nonumber \\
& \quad  \quad   +  u_i(s),\;\; \forall s\in \mathcal{S},  \nonumber
\end{align}
where from uniformization one has $ \check p(j|s, a_s^+) = \tilde p(j|s, a)  \tilde \lambda(s,a)/\lambda_{\max}$ for $j \ne s$, $ \check p(s |s, a_s^+) = 1- (1- \tilde p(s|s, a)) \tilde \lambda(s,a)/\lambda_{\max}$, and $\check r(s, a_s^+) = r(s, a) \tilde \lambda(s,a)  /\lambda_{\max}$. }
If one stops the value iteration whenever
\begin{align}\label{eq:stop}
\max_{s \in \mathcal{S}} \{ u_{i+1}(s)  - u_i(s)\} - \min_{s \in \mathcal{S}} \{ u_{i+1}(s)  - u_i(s)\} <  \varepsilon := \frac{1}{\sqrt{t_k}},
\end{align}
 then the greedy policy with respect to $u_i$ is $\varepsilon-$optimal for the extended CTMDP $\widetilde M_k^+$ which yields the nearly optimistic policy $\tilde \pi_k$ in \eqref{eq:optimistic-model}. The optimistic CTMDP $\widetilde M_k =( \widetilde p_k (\cdot | s,a) , \widetilde \lambda_k(s,a))_{s\in \mathcal{S}, a \in \mathcal{A}}$ is obtained by solving the two inner optimizations in \eqref{eq:EVI}. \rev{Note that the confidence set $\mathcal{P}_k(s,a)$ of transition probabilities is a convex polytope}, and
the optimization over $\mathcal{P}_k(s,a)$ can be solved using the algorithm in Figure 2 of \cite{Jaksch2010}. Given the solution to this problem and $u_i$, one can easily solve the optimization over $\tilde \lambda(s, a) $ since the objective is linear in $\tilde \lambda(s,a)$ and the domain $C_k(s,a)$ is bounded with $\tilde \lambda(s,a) \in [\lambda_{\min}, \lambda_{\max}].$
\begin{remark}
\rev{
The convergence of extended value iteration for solving an extended CTMDP follows directly from an argument similar to that in the proof of Theorem 7 of \cite{Jaksch2010}. This is because the uniformization procedure turns an extended CTMDP to an equivalent extended (discrete-time) MDP.
Though the extended CTMDP and its uniformized MDP have continuous action spaces, in the extended value iteration \eqref{eq:EVI} one can restrict the optimization over confidence sets of transition probabilities and holding time rates to finitely many extremal points as discussed earlier.
Following the discussion in Section 3.1.3 of  \cite{Jaksch2010}, in each iteration of this extended value iteration, there exists some state $s'$ such that the policy chosen (say $\pi$) will have positive transition probability from $s'$ to $s'$, i.e. $\tilde p(s'|s', a')>0$ where $a'$ is the action chosen under policy $\pi$ at state $s'$. By uniformization, it follows that $ \check p(s' |s', a_{s'}^+) = 1- (1- \tilde p(s'|s', a')) \tilde \lambda(s',a')/\lambda_{\max}>0.$
This suggests that extended value iteration chooses a policy such that the uniformized discrete-time MDP has an aperiodic transition matrix.
Hence, the extended value iteration converges as argued in the proof  of \cite[Theorem 7]{Jaksch2010}.
}

\end{remark}

\begin{remark}
The CT-UCRL algorithm requires the values of $\lambda_{\min}$ and $\lambda_{\max}$ as input. If these parameters are unknown,
one possible remedy is to dedicate
the beginning of a learning process to estimating such unknown parameters. 
In our problem, because $1/\lambda_{\max}\le 1/\lambda(s,a) \le 1/\lambda_{\min}$ for all $(s,a )$, we may apply a randomized policy at the beginning of the horizon (which may incur regret proportional to the time spent in doing so), and
use the minimum and maximum of the estimated mean holding times for different state-action pair $(s,a)$ to replace $1/\lambda_{\max}$ and $1/\lambda_{\min}$ respectively in implementing the CT-UCRL algorithm for the rest of the horizon.
A more sophisticated strategy is to adaptively update the estimates over different episodes when more data have been collected. Such strategies are not difficult to implement, but their theoretical analysis is challenging. This is because  $\lambda_{\min}$ and $\lambda_{\max}$ appear in the holding time rate estimator and its confidence bound, and the CT-UCRL algorithm also computes the optimistic CTMDP over the confidence balls in each episode. Errors in the estimation of $\lambda_{\min}$ and $\lambda_{\max}$ may then propagate to these steps, and hence it is nontrivial to analyze the regret for such procedures. This is an interesting research question and we leave it to the future.
\end{remark}


\subsection{Regret upper bound} \label{sec:result-UB}

We now present the logarithmic upper bound of the CT-UCRL algorithm (Algorithm~\ref{alg: CTUCRL}). Recall $\rho^\pi_s(\mathcal{M})$ and $\rho^*(\mathcal{M})$ defined in equation~\eqref{eq:ergodic-control}. Set
\begin{align} \label{eq:gap-g}
g = \rho^*(\mathcal{M}) - \max_{s \in \mathcal{S}} \max_{\pi : \mathcal{S} \rightarrow \mathcal{A}} \{\rho^\pi_s(\mathcal{M}): \rho^\pi_s(\mathcal{M}) < \rho^*(\mathcal{M}) \} ,
\end{align}
which is the gap in the average reward between the best and second best (deterministic) policy in the CTMDP $\mathcal{M}.$ \rev{We have the following result on logarithmic regret of the CT-UCRL algorithm when the total number of decision epochs $N$ is known. } 

\begin{theorem}\label{thm:UCRL}
For any CTMDP $\mathcal{M} \in \mathcal{H}$, any initial state $s_0$ and any $N \ge 1$, \rev{by setting $\delta = 1/N$ in Algorithm~\ref{alg: CTUCRL},}  the expected regret of the algorithm $\bar{\mathcal{G}}:=\text{CT-UCRL}$ up to the $N$-th decision epoch satisfies
\begin{align} \label{eq:bound-RN-UCRL}
 \mathcal{R}_N^{\bar{\mathcal{G}}}(s_0, \mathcal{M}) \le  4C \cdot \frac{\log(N)}{g} + C',
\end{align}
where $C'$ is a constant independent of $N$ and
\begin{align}\label{eq:upper-C}
C =3  \left(34^2  \lambda_{\max}^2 D(\mathcal{M})^2 S^2 A  + 2\cdot 73^2  \frac{ \lambda_{\max}^2 }{ \lambda_{\min}^2} SA  +  \rev{  \frac{6 SA (1+ \lambda_{\max})^2}{\lambda_{\min}^2} } \right)
\end{align}
 with $D(\mathcal{M}) < \infty$ being the diameter of the  CTMDP $\mathcal{M}.$
\end{theorem}

\rev{In the above theorem, the algorithm input parameter $\delta$ depends on the total number of decision epochs $N$. To have an algorithm that has a logarithmic regret bound for {\it all} $N$, one can use the so-called (exponential) doubling trick \cite{auer2010ucb, besson2018doubling} to adaptively set the parameter $\delta$. The algorithm, which is henceforth referred to as the ``doubling CT-UCRL algorithm", can be described as follows. Consider an increasing horizon sequence $N_i = 2^{2^{i}}$ for $i \ge 0$ and set $N_{-1}=0$. For each $i \ge -1,$ one restarts
the CT-UCRL algorithm at the beginning of the decision step $N_i +1$, and run the
algorithm until the decision step $N_{i+1}$ with the input parameter $\delta_i = 1/(N_{i+1} - N_i)$. Then we have the following result.}

\begin{corollary}\label{thm:UCRL2}
\rev{For any CTMDP $\mathcal{M} \in \mathcal{H}$, any initial state $s_0$ and any $N \ge 1$,  the expected regret of the doubling CT-UCRL algorithm $\bar{\mathcal{G}}^{DT}$ satisfies
\begin{align} \label{eq:bound-RN-UCRL2}
 \mathcal{R}_N^{\bar{\mathcal{G}}^{DT}}(s_0, \mathcal{M}) \le  \frac{16C}{g}  \log N  + [\log_2 \log_2 N + 2 ] \cdot C' .
\end{align}
Moreover
\begin{align}\label{eq: coro:ub}
 \mathbb{E}_{s_0}[ R_T^{\bar{\mathcal{G}}^{DT}}(s_0, \mathcal{ M})] \le     \frac{16C}{g} \cdot \log \left( {\lambda_{\max} T +2}\right) + C' \cdot [ 2+ \log_2 \log_2 ( \lambda_{\max} T +2)]  + \frac{\lambda_{\max}}{\lambda_{\min}} +1.
\end{align} }
\end{corollary}

\rev{Proofs of the two results above are 
given in Section~\ref{sec:N-UB} and \ref{sec:N-UB2} respectively.
\begin{remark}
Theorem~\ref{thm:UCRL} stipulates  that the instance-dependent regret upper bound of  the CT-UCRL algorithm for learning CTMDPs is given by
\begin{align}\label{eq:O-Regret}
O \left( \frac{ \lambda_{\max}^2 D(\mathcal{M})^2 S^2 A +  \left[\frac{ \lambda_{\max}^2 +1 }{ \lambda_{\min}^2} \right]SA }{g} \cdot \log N \right).
\end{align}
In the discrete-time setting,
 \cite{Jaksch2010} develop the UCRL2 algorithm and show that its instance-dependent regret bound is $$O \left( \frac{D^2 S^2 A\log N}{g^*} \right),$$
where $D$ is the diameter of the discrete-time MDP and $g^*$ is the gap parameter that can be defined similarly to \eqref{eq:gap-g} for discrete-time MDPs. Note that the gap parameter is instance-dependent  and there is no simple explicit relation between it and the model parameters in general even for discrete-time MDPs. However, we can still observe some similarities and differences in the regret upper bounds for continuous-time and discrete-time MDPs. The term $\lambda_{\max} D(\mathcal{M})$ in \eqref{eq:O-Regret} appears naturally because it is the diameter of the uniformized discrete-time MDP; see Remark \ref{remark:diam}. The term $\left[\frac{ \lambda_{\max}^2 +1 }{ \lambda_{\min}^2} \right]SA$ in \eqref{eq:O-Regret} arises from  the random exponential holding times.
\end{remark}
}

\begin{remark}
Our bound \eqref{eq:O-Regret} has an $\tilde O(S^2 A)$-dependence on the sizes of the
state and action spaces. On the other hand, it has been shown in \cite{yang2021q} that the optimistic Q-learning algorithm in \cite{jin2018q} has regret $O(\frac{SA H^6}{\Delta_{\min}} \log(SAT))$ for finite horizon  discrete-time episodic MDPs, where $T$ is the total number of steps, $H$ is the time horizon and $\Delta_{\min}$ is the minimum sub-optimality gap of the optimal $Q-$function. In the same episodic setting, \cite{xu2021fine} obtain a lower bound of $\Omega(\frac{SA }{\Delta_{\min}} \log(T))$.
Our $\tilde O(S^2A)$-bound is
due to the specific algorithm and approach we develop to solve our problem, and it could be suboptimal. It is an enormously
interesting open question whether we can obtain non-asymptotic logarithmic regrets with $\tilde O(SA)$ dependance in the average-reward setting for continuous-time MDPs. This may require the development of new proof techniques to obtain a tighter analysis of our proposed algorithm or the development of entirely new algorithms. We leave this study  to the future.
\end{remark}


\subsection{Proof of Theorem~\ref{thm:UCRL}}\label{sec:N-UB}
The proof is lengthy; we divide it into several steps. In Section~\ref{sec:CI} we study failing confidence regions. In Section~\ref{sec:UB_UCRL} we bound the number of suboptimal decision steps of the CT-UCRL algorithm. With such a bound, we then prove Theorem~\ref{thm:UCRL} in Section~\ref{sec:thmUCRLproof}.

\subsubsection{Failing confidence regions} \label{sec:CI}

\begin{lemma}[Failing confidence region] \label{lem:failCI}
For any episode $k \ge 1$, the probability that the true CTMDP $\mathcal{M}$ is not contained in the set of plausible CTMDPs $\mathcal{C}_k$ is at most $\frac{\delta}{15 t_k^6}$, i.e.
\begin{align} \label{eq:fail-CI}
\mathbb{P}(\mathcal{M} \notin \mathcal{C}_k) \le \frac{\delta}{15 t_k^6}.
\end{align}
\end{lemma}
\begin{proof}
First, from the proof of Lemma 17 in \cite{Jaksch2010}, we have
\begin{align} \label{eq:CI-p}
\mathbb{P} \left( || p(\cdot|s, a) -  \hat p(\cdot|s, a)  ||_1 > \sqrt{\frac{14S \log(2At_k/\delta)}{ \max\{1, N_k(s,a)\}} } \right) \le  \frac{\delta}{20 t_k^6 SA}.
\end{align}
Second, it follows from \eqref{eq:b1} that given $n$ i.i.d samples of holding times at $(s,a)$, the truncated mean estimator $\frac{1}{\hat \lambda(s,a)}$ satisfies
\begin{align*}
\mathbb{P} \left( \left|\frac{1}{\hat \lambda(s,a)} - \frac{1}{\lambda(s,a)} \right| \ge \frac{ 4 }{\lambda_{\min}} \cdot  \sqrt{ \frac{14}{n} \log(2 AS t/\delta)} \right) \le  \frac{\delta}{60 t^7 SA},\;\;\forall t\ge 1.
\end{align*}
Since $N_k(s,a) \le t_k -1$, we can use a union bound over all possible values of $N_k(s,a)=1, 2, \ldots, t_k-1$ and obtain
\begin{align} \label{eq:holding-time-cb}
&\mathbb{P} \left( \left|\frac{1}{\hat \lambda(s,a)} - \frac{1}{\lambda(s,a)} \right| \ge  \frac{4 }{\lambda_{\min}} \sqrt{ \frac{14 \log(2 AS t_k/\delta)}{ \max\{1, N_k(s,a)\}}} \right)  \\
& \rev{ =   \sum_{i=1}^{t_k-1} \mathbb{P} \left( \left|\frac{1}{\hat \lambda(s,a)} - \frac{1}{\lambda(s,a)} \right| \ge  \frac{4 }{\lambda_{\min}} \sqrt{ \frac{14 \log(2 AS t_k/\delta)}{ \max\{1, N_k(s,a)\}}},  N_k(s,a) =i \right) } \nonumber \\
&  \le  \sum_{i=1}^{t_k-1}\frac{\delta}{60 t_k^7 SA} \le \frac{\delta}{60 t_k^6 SA}. \nonumber
\end{align}
In view of \eqref{eq:CI-p} and \eqref{eq:holding-time-cb}, we can sum over all state-action pairs and obtain \eqref{eq:fail-CI}.
\end{proof}


\subsubsection{Bounding the number of suboptimal decision steps of CT-UCRL}\label{sec:UB_UCRL}

In this section, we establish a bound on the number of decision steps in suboptimal episodes for the CT-UCRL algorithm. This bound is critical in the proof of Theorem~\ref{thm:UCRL}.

Fix the total number of decision steps $N$. Denote
\begin{align} \label{eq:deltak}
\Delta_k = \sum_{(s,a)} v_k(s,a) ( \rho^*/\lambda(s,a) - r(s,a)),
\end{align}
where $v_k(s,a)$ denotes the number of visits of the CTMDP to the state-action pair $(s,a)$ in episode $k$, up to step $N$. We call $\Delta_k$ the \textit{adjusted regret} in episode $k$, and
we will see later that conditional on $(v_k(s,a))_{(s,a)}$, the expected value of
$\Delta_k$ corresponds to the expected regret incurred in episode $k$.
We say that an episode $k$ is $\epsilon$-bad if $\frac{\Delta_k}{\ell_k} \ge \epsilon$, where $l_k$ is the number of decision steps in episode $k$.

We are to provide a bound on the number of decision steps in $\epsilon$-bad episodes for the CT-UCRL algorithm. Theorem~11 in  \cite{Jaksch2010} establishes such a bound for the discrete-time MDP. There are essential difficulties in extending its proof to the continuous-time setting due to  exponential holding times and a different Bellman optimality equation~\eqref{eq:opt-eqn}. To overcome these difficulties, we need two new ingredients in our analysis. The first one is the high probability confidence bound for the rate of the holding times; see \eqref{eq:holding-time-cb}.
The second one is the analysis of the extended value iteration \eqref{eq:EVI} for CTMDPs. 

The following is the main result in this subsection.

\begin{proposition} \label{lem:badepisodes}
Let $L_{\epsilon}(N)$ be the number of decision steps taken by CT-UCRL in $\epsilon-$bad episodes up to step $N$. Then for any initial state $s,$ $N \ge2 SA, \epsilon>0$ and $\delta \in (0, 1/2)$, with probability at least $1-2 \delta,$
\begin{align}\label{eq:LepsN}
L_{\epsilon}(N) \le  \frac{3}{ \epsilon^2} \left(34^2  \lambda_{\max}^2 D(\mathcal{M})^2 S^2 A  + 2\cdot 73^2  \frac{ \lambda_{\max}^2 }{ \lambda_{\min}^2} SA  + \rev{ \frac{ 6 SA (1+ \lambda_{\max})^2}{\lambda_{\min}^2} } \right) \cdot \log\left(\frac{N}{\delta}\right) ,
\end{align}
where $D(\mathcal{M})$ is the diameter of the true CTMDP $\mathcal{M}.$
\end{proposition}

\begin{proof}
Fix  $N>1.$
Denote by $K_{\epsilon}$ the random set that contains the indices of the $\epsilon$-bad episodes up to step $N$.
Set $\Delta'(\epsilon, N) =  \sum_{k \in K_{\epsilon}} \Delta_k$, and note $L_{\epsilon}(N):= \sum_{k \in K_{\epsilon}} \sum_{(s,a)} v_k(s,a).$

Following the same argument as on p.1580 of  \cite{Jaksch2010}, we can infer from Lemma~\ref{lem:failCI} that $
\mathbb{P} \left( \sum_{k \in K_{\epsilon}} \Delta_k 1_{\mathcal{M} \notin \mathcal{C}_k}>0 \right) \le \delta,$
where the confidence set $\mathcal{C}_k$ is the set of all plausible CTMDPs in episode $k.$
Hence, with probability at least $1-\delta,$
\begin{align} \label{eq:d-eps}
\Delta'(\epsilon, N) \le \sum_{k \in K_{\epsilon}} \Delta_k 1_{\mathcal{M} \in \mathcal{C}_k}.
\end{align}

Next we bound $ \Delta_k 1_{\mathcal{M}\in \mathcal{C}_k }$.
By \eqref{eq:optimistic-model} and the assumption that $\mathcal{M}\in \mathcal{C}_k,$ we have
 $\rho^* \le \rho^*_k + \frac{1}{ \sqrt{t_k}}$, where $\rho^*$ is the optimal average reward of $\mathcal{M}$.
Then
\begin{align} \label{eq:delta-k}
\Delta_k & = \sum_{ (s,a)}  v_k(s,a) ( \rho^*/\lambda(s,a) - r(s,a)) \nonumber \\
& \le \sum_{ (s,a)} v_k(s,a) ( \rho^*_k /\lambda(s,a) - r(s,a))  + \sum_{ (s,a)} \frac{v_k(s,a) }{ \lambda(s,a) \sqrt{t_k} } \nonumber \\
    & \le \sum_{ (s,a)} v_k(s,a) ( \rho^*_k /\widetilde \lambda_k(s,a) - r(s,a)) + \sum_{ (s,a)} v_k(s,a) ( \rho^*_k /\lambda(s,a) - \rho^*_k /\widetilde  \lambda_k(s,a) ) \nonumber \\
    & \quad \quad \quad  + \sum_{ (s,a)} \frac{v_k(s,a) }{ \lambda_{\min} \sqrt{t_k} } ,
\end{align}
where  the last inequality follows from the fact that $\lambda(s,a) \ge \lambda_{\min}$ for all $s,a.$
Note in episode $k,$ we execute the optimistic policy $\tilde \pi_k,$ i.e., $a=\tilde \pi_k(s)$. Hence $v_k(s,a)= 0$ if $a \ne \tilde \pi_k(s)$.

\rev{We now analyze the first term of the right hand side of \eqref{eq:delta-k}. 
One can directly deduce from uniformization of CTMDPs and Theorem 8.5.6 of \cite{puterman2014markov} that when the convergence criterion of value iteration in \eqref{eq:stop} is met at iteration $i$, we have
\begin{align}\label{eq:vi-stop}
|u_{i+1}(s) - u_i(s) - \frac{\rho^*_k}{\lambda_{\max}}| \le \frac{1}{\sqrt{t_k}} \quad \text{for all $s$,}
\end{align}
where $\rho^*_k$ is the average reward of the policy $\tilde \pi_k$ chosen in this iteration on the optimistic CTMDP $\widetilde M_k =( \widetilde p_k (j | s,a) , \widetilde \lambda_k(s,a))_{s, j \in \mathcal{S}, a \in \mathcal{A}}$ in episode $k$. In addition, we can directly obtain from
\eqref{eq:EVI} that
\begin{eqnarray*}
u_{i+1}(s) - u_i(s) =   \left(   r(s, \tilde \pi_k(s))  +  \sum_{j \in \mathcal{S}} \widetilde p_k (j | s, \tilde \pi_k(s)) u_i(j) - u_i (s) \right) \cdot \frac{  \widetilde \lambda_k (s, \tilde \pi_k(s))}{  \lambda_{\max}}.
\end{eqnarray*}
Together with \eqref{eq:vi-stop}, we have
for all $s \in \mathcal{S},$
\begin{eqnarray*}
\left| \rho^*_k - \left( r(s, \tilde \pi_k(s)) +  \sum_{j \in \mathcal{S}} \widetilde p_k (j | s, \tilde \pi_k(s)) u_i(j) - u_i (s) \right) \cdot {  \widetilde \lambda_k (s, \tilde \pi_k(s))}\right| \le \frac{\lambda_{\max}}{\sqrt{t_k}}.
\end{eqnarray*}
Because $\widetilde \lambda_k (s, \tilde \pi_k(s)) \ge \lambda_{\min}$, the first term in \eqref{eq:delta-k} is bounded above by
\begin{align}\label{eq:v-w}
&  \frac{\lambda_{\max}}{ \lambda_{\min}}\sum_{s \in \mathcal{S}} v_k(s, \tilde \pi_k(s) ) \frac{1}{\sqrt{t_k} }  + \sum_{s \in \mathcal{S}} v_k(s, \tilde \pi_k(s) ) \left(  \sum_{j \in \mathcal{S}} \widetilde p_k (j | s, \tilde \pi_k(s)) u_i (j) - u_i(s) \right)  \nonumber \\
  &=   \frac{\lambda_{\max}}{ \lambda_{\min}} \sum_{ (s,a)} \frac{v_k(s,a) }{\sqrt{t_k} } +  \mathbf{v}_k (\widetilde  P_k - I) w_k,
\end{align}
where  $\mathbf{v}_k:= (v_k(s, \tilde \pi_k(s)))_s$, $\widetilde  P_k :=( \widetilde p_k (j | s, \tilde \pi_k(s)))_{s,j}$ is the transition matrix of $\tilde \pi_k$ on $\widetilde M_k$,  $w_k= (w_k(s))_s$ with
$
w_k(s) = u_i (s) - \frac{ \max_{s \in \mathcal{S}}u_i (s) + \min_{s \in \mathcal{S} } u_i(s) }{2},
$
and  \eqref{eq:v-w} holds because the matrix $\widetilde P_k - I$ right-multiplied by a constant vector is zero.
By Lemma~6 of \cite{fruit2017exploration}, $||w_k||_{\infty}$ is bounded by $D/2$ where  $D:={D(\mathcal{M} ) \lambda_{\max} }$ for all $k$. }
Note that
\begin{align}\label{eq:1st-1}
 \mathbf{v}_k ( \widetilde P_k - I) w_k =   \mathbf{v}_k ( \widetilde P_k - P_k ) w_k  +  \mathbf{v}_k ( P_k - I) w_k ,
\end{align}
where $P_k$ is the transition matrix associated with the policy $\tilde \pi_k$ under the true model. By H\"{o}lder's inequality, we deduce
\begin{align} \label{eq:1st-2}
 \mathbf{v}_k ( \widetilde P_k - P_k ) w_k &\le \sum_{s} v_k(s, \tilde \pi_k(s)) || \widetilde P_k (\cdot| s, \pi_k(s))- P_k (\cdot| s, \tilde \pi_k(s))  ||_1 \cdot ||w_k||_{\infty} \nonumber \\
& \le   \sum_{s} v_k(s, \tilde \pi_k(s)) \sqrt{\frac{14S \log(2At_k/\delta)}{ \max\{1, N_k(s, \tilde \pi_k(s))\}} } \cdot D/2 \nonumber \\
& \le \frac{D}{2} \sqrt{ 14S \log(2AN/\delta) } \sum_{(s,a)} \frac{v_k(s,a)}{ \sqrt{\max\{1,N_{k}(s,a) \}} },
\end{align}
where the second inequality  follows from the fact that the transition matrices of both the optimistic model and the true model lie in the confidence region \eqref{eq:P-CI}, while the third inequality uses the fact that $t_k \le N.$

For the second term of the right hand side of \eqref{eq:delta-k}, we first note that the reward function $r(s,a)$ is bounded by 1 for all $(s,a)$; so the long-run average reward per unit time $\rho^*_k \le \lambda_{\max} $. Hence
\begin{align*}
 \rho^*_k /\lambda(s,a) - \rho^*_k / \widetilde \lambda_k(s,a) & \le \lambda_{\max}  |1/\lambda(s,a) -1 /  \widetilde\lambda_k(s,a) |  \\
& \le \lambda_{\max} \left(  |1/\lambda(s,a) -1 /\hat \lambda_k(s,a) | + |1/\hat \lambda_k (s,a) -1 /  \widetilde \lambda_k(s,a) | \right)\\
& \le \frac{ 8 \lambda_{\max} }{ \lambda_{\min}}  \cdot \sqrt{ \frac{14 \log(2 AS t_k/\delta)}{ \max\{1, N_k(s,a)\}}},
\end{align*}
where we use the confidence region \eqref{eq:rate-CI} for the rates of holding times and the fact that $\mathcal{M}, \widetilde M_k \in \mathcal{C}_k$. Thus the second term of \eqref{eq:delta-k} is upper bounded by
\begin{align}\label{eq:2nd}
 & \sum_{ (s,a)}  v_k(s,a) ( \rho^*_k /\lambda(s,a) - \rho^*_k /  \widetilde \lambda_k(s,a) ) \nonumber \\
& \le
 \frac{ 8 \lambda_{\max} }{ \lambda_{\min}}  \cdot   \sum_{s, a} v_k(s,a)  \sqrt{ \frac{14 \log(2 AS t_k/\delta)}{ \max\{1, N_k(s,a)\}}} \nonumber \\
& \le  \frac{ 8 \lambda_{\max} }{ \lambda_{\min}}  \cdot \sqrt{ 14 \log(2 AS N/\delta)}  \cdot  \sum_{s, a} \frac{  v_k(s,a) }{ \sqrt{\max\{1, N_k(s,a)\}} } .
\end{align}

Noting $t_k \ge \max\{1, N_k(s,a) \}$ for all $s,a$, and
combining \eqref{eq:delta-k}, \eqref{eq:v-w}, \eqref{eq:1st-1}, \eqref{eq:1st-2}, and \eqref{eq:2nd}, we obtain for episode $k$ with $\mathcal{M} \in \mathcal{C}_k,$
\rev{\begin{align*}
\Delta_k & \le \frac{D}{2}  \sqrt{ 14S \log(2AN/\delta) } \cdot   \sum_{(s,a)} \frac{v_k(s,a)}{ \sqrt{\max\{1,N_{k}(s,a) \}} }  +   \mathbf{v}_k ( P_k - I) w_k\\
& \quad +   \frac{ 8 \lambda_{\max} }{ \lambda_{\min}}  \cdot \sqrt{ 14 \log(2 AS N/\delta)}  \cdot  \sum_{(s,a)} \frac{v_k(s,a)}{ \sqrt{\max\{1,N_{k}(s,a) \}} }  + \frac{(1 + \lambda_{\max})}{\lambda_{\min}}  \sum_{(s,a)} \frac{v_k(s,a)}{ \sqrt{\max\{1,N_{k}(s,a) \}} }.
\end{align*} }
From  equation (27) of  \cite{Jaksch2010}, we have
\begin{align*}
 \sum_{k \in K_{\epsilon}}\sum_{(s,a)} \frac{v_k(s,a)}{\max\{1,N_{k}(s,a) \}} \le (1 + \sqrt{2}) \sqrt{L_{\epsilon}(N) SA}.
\end{align*}
It follows from \eqref{eq:d-eps} that with probability at least $1- \delta,$
\begin{align} \label{eq:delta_N}
 \Delta'(\epsilon, N)   
&\le \sum_{k \in K_{\epsilon}} \mathbf{v}_k ( P_k - I) w_k 1_{ \mathcal{M} \in \mathcal{C}_k }   +  \frac{(1 + \lambda_{\max})}{\lambda_{\min}}    \cdot (1 + \sqrt{2}) \sqrt{L_{\epsilon}(N) SA}  \nonumber  \\
& \quad +  \left(   \left[\frac{D}{2} +     \frac{ 8 \lambda_{\max} }{ \lambda_{\min} } \right] \cdot  \sqrt{ 14 \log(2 AS N/\delta)}    \right)      \cdot (1 + \sqrt{2}) \sqrt{L_{\epsilon}(N) SA}
\end{align}
\rev{To bound the first term on the right hand side of  the above equation, we recall that $P_k$ is the true transition matrix of the policy $\tilde \pi_k$ applied in episode $k$ under the true CTMDP model; so it is just the transition matrix of the embedded discrete-time Markov chain. Hence,
we can use the same martingale argument as in \cite[pp. 1580--1581]{Jaksch2010}  for discrete-time MDPs to obtain that, with probability at least $1- \delta$,
\begin{align}\label{eq:mart-1}
 \sum_{k \in K_{\epsilon}} \mathbf{v}_k ( P_k - I) w_k 1_{\mathcal{M} \in \mathcal{C}_k } \le 2 D \sqrt{L_{\epsilon} (N) \log(N/\delta)} +  DSA \log_2 (8N/SA).
\end{align}
Specifically, to establish \eqref{eq:mart-1}, one can introduce the martingale difference sequence $(X_i)$ where $X_i= \left(p(\cdot |s_i, a_i) - e_{s_{i+1}} \right) w_{k(i)} 1_{ \{i \in J_{\epsilon}, \mathcal{M} \in \mathcal{C}_k \}} $, where $k(i)$ denotes the episode containing the decision step $i$, $a_i =\tilde \pi_{k(i)} (s_i) $, and $e_{j}$ is the standard unit vector with $j-$th entry 1 and other entries 0. One can verify that
\begin{align*}
\sum_{k \in K_{\epsilon}} \mathbf{v}_k ( P_k - I) w_k 1_{\mathcal{M} \in \mathcal{C}_k } \le \sum_{i \in J_{\epsilon}} X_i + DSA \log_2 (8N/SA),
\end{align*}
where $D={D(\mathcal{M} ) \lambda_{\max} }$ bounds $2 ||w_k||_{\infty}$ and $SA \log_2 (8N/SA) $ bounds the number of episodes of CT-UCRL up to the decision step $N$ (see \cite[Proposition 18]{Jaksch2010}). The term $\sum_{i \in J_{\epsilon}} X_i$ can be bounded by appling the Freedman--Bernstein inequality for martingales, which then leads to \eqref{eq:mart-1}. As the argument is the same as in \cite{Jaksch2010} we omit further details.}

Now we can infer from \eqref{eq:delta_N} that with probability at least $1- 2\delta$,
\begin{align*}
&\Delta'(\epsilon, N) \\
&\le 2 D(\mathcal{M} ) \lambda_{\max} \sqrt{L_{\epsilon}(N) \log(N/\delta)} + D(\mathcal{M} ) \lambda_{\max} SA \log_2 (8N/SA) \\
& \quad +       \left(  \frac{D(\mathcal{M} ) \lambda_{\max}}{2 } \sqrt{ 14S \log(2AN/\delta) } +     \frac{ 8 \lambda_{\max} }{ \lambda_{\min}}  \cdot  \sqrt{ 14 \log(2 AS N/\delta)}    \right)    \cdot  (1 + \sqrt{2}) \sqrt{L_{\epsilon}(N) SA}\\
  & \quad   + \rev{ \frac{(1+ \lambda_{\max})}{\lambda_{\min}} }  \cdot (1 + \sqrt{2}) \sqrt{L_{\epsilon}(N) SA}.
\end{align*}
By equation~(32) of   \cite{Jaksch2010}, this can be simplified to
\begin{align}\label{eq:regret-epsN}
& \Delta'(\epsilon, N) \nonumber \\
&\le 34  D(\mathcal{M} ) \lambda_{\max} S \sqrt{L_{\epsilon}(N) A \log(N/\delta)} + \frac{ 8 \lambda_{\max} }{ \lambda_{\min}}  \cdot  \sqrt{ 14 \log(2 AS N/\delta)}    (1 + \sqrt{2}) \sqrt{L_{\epsilon}(N) SA} \nonumber \\
  & \qquad   \qquad  + \rev{ \frac{(1+ \lambda_{\max})}{\lambda_{\min}} } \cdot (1 + \sqrt{2}) \sqrt{L_{\epsilon}(N) SA}.
\end{align}
Noting by definition $ \Delta'(\epsilon, N) =\sum_{k \in K_{\epsilon}} \Delta_k \ge \epsilon \sum_{k \in K_{\epsilon}} \ell_{k} =\epsilon L_{\epsilon}(N)$, and $8 (1+ \sqrt{2}) \sqrt{14} \le 73$, we get
\begin{align*}
L_{\epsilon}(N) \le  \frac{3}{ \epsilon^2} \left(34^2   \lambda_{\max}^2 D(\mathcal{M})^2 S^2 A  \log\left(\frac{N}{\delta}\right)  + 73^2  \frac{ \lambda_{\max}^2 }{ \lambda_{\min}^2} SA   \log\left(\frac{2ASN}{\delta}\right) + \rev{ \frac{ 6 SA (1+ \lambda_{\max})^2 }{\lambda_{\min}^2} } \right) .
\end{align*}
For $N \ge 2SA,$ $\log\left(\frac{2ASN}{\delta}\right) \le 2 \log\left(\frac{N}{\delta}\right)$. In addition, $\log\left(\frac{N}{\delta}\right) \ge 1$ for $N \ge 2$ and $\delta<1/2.$ Consequently, we obtain \eqref{eq:LepsN}.
\end{proof}



\subsubsection{Proof of Theorem~\ref{thm:UCRL}}\label{sec:thmUCRLproof}

We adapt the proof of Theorem~4 in \cite{Jaksch2010} to our continuous-time setting. There are two new ingradients in our analysis. First, we connect the \textit{adjusted regret} (see Equation~\eqref{eq:deltak}) with the actual regret occured in each episode; Second, we apply stochastic comparison of point processes to convert the regret upper bound in number of decision steps to one in the time horizon.

We first prove \eqref{eq:bound-RN-UCRL}.
From Proposition~\ref{lem:badepisodes},  it follows that with probability at least $1 - 2 \delta$, $L_{\epsilon}(N) \le  \frac{C  \log(N/\delta)}{ \epsilon^2}$ for $N \ge 2SA$, where
\begin{align*}
C =3  \left(34^2  \lambda_{\max}^2 D(\mathcal{M})^2 S^2 A  + 2\cdot 73^2  \frac{ \lambda_{\max}^2 }{ \lambda_{\min}^2} SA  +  \rev{  \frac{ 6 SA (1+ \lambda_{\max})^2}{\lambda_{\min}^2} } \right).
\end{align*}

Then we can infer from \eqref{eq:regret-epsN} that
 the adjusted regret accumulated in $\epsilon$-bad episodes is bounded by
\begin{align*}
\Delta'(\epsilon, N) = \sum_{k \in K_{\epsilon}} \Delta_k &\le C \cdot \frac{\log(N/\delta)}{\epsilon}
\end{align*}
with probability at least $1 - 2 \delta.$ By Assumption~\ref{assume-rate} and the fact that $r(s,a) \in [0,1]$,
we have the following simple bound:
\begin{align} \label{eq:per-step}
\rho^*/\lambda(s,a) - r(s,a) \le \rho^*/\lambda_{\min} \le \lambda_{\max}/ \lambda_{\min}.
\end{align}
Hence, by choosing $\delta = 1/N$, we can bound the expected adjusted regret in $\frac{g}{2}$-bad episodes of the algorithm $\bar{\mathcal{G}}:=\text{CT-UCRL}$ up to decision step $N $  as follows:
\begin{align}\label{eq:R1}
\mathbb{E}_{s_0}\left[\Delta'(\frac{g}{2}, N) \right] = \mathbb{E}_{s_0} \left[ \sum_{k \in K_{\frac{g}{2}}} \Delta_k \right] &\le 4C \cdot \frac{\log(N)}{g} + \frac{2 \lambda_{\max} }{\lambda_{\min}}.
\end{align}

It remains to bound the regret in those episodes $k$ with average adjusted regret smaller than $\frac{g}{2}$, i.e. $k \notin K_{\frac{g}{2}}.$  To this end, note that for each policy $\pi: \mathcal{S} \rightarrow \mathcal{A}$, there is $n_{\pi}$ such that for all $n \ge n_{\pi}$ the expected average reward after $n$ decision steps is $\frac{g}{2}$-close to the average reward of $\pi$. Then applying the same argument as in the proof of Theorem~4 in \cite{Jaksch2010},
we have
\begin{align*}
 \mathbb{E}_{s_0} \left[ \sum_{k \notin K_{\frac{g}{2}}} \Delta_k \right] &\le  \frac{\lambda_{\max}}{\lambda_{\min}} \cdot C_1, \quad \text{where $C_1:= \sum_{s,a} [1+ \log_2 (\max_{\pi: \pi(s)=a}n_{\pi})] \cdot \max_{\pi: \pi(s)=a} n_{\pi}$},
\end{align*}
where the factor $ \frac{\lambda_{\max}}{\lambda_{\min}}$ arises  due to the bound in \eqref{eq:per-step} in our setting.
Combining with \eqref{eq:R1}, we get the expected regret of UCRL to be
\begin{align} \label{eq:e-regre1}
 \mathbb{E}_{s_0} \left[ \sum_{k=1}^m \Delta_k \right] &\le  4C \cdot \frac{\log(N)}{g} + \frac{\lambda_{\max}}{\lambda_{\min}} (2 + C_1),
\end{align}
where $\Delta_k$ is defined in \eqref{eq:deltak}, and $m$ is the number of episodes started up to step $N$.

Note, however, that the expected regret of the CT-UCRL algorithm $\bar{\mathcal{G}}$ up to  step $N$ is given by
\begin{eqnarray*}
\mathcal{R}_N^{\bar{\mathcal{G}}}(s_0, \mathcal{M})=\mathbb{E}_{s_0}[ \rho^* \cdot S_N - \sum_{n=0}^{N-1} r(s_n, a_n)] \le \mathbb{E}_{s_0}\left[ \sum_{k=1}^m \hat \Delta_k \right] ,
\end{eqnarray*} 
where 
$
\hat \Delta_k = \sum_{(s,a)}\sum_{j=1}^{v_k(s,a)} ( \rho^* \tau_{k_j}(s,a) - r(s,a)),
$
and $(\tau_{k_j}(s,a))_{j}$ are i.i.d. exponential random variables with rate $\lambda(s,a)$.
The difference compared with \eqref{eq:e-regre1} is
$\mathbb{E}_{s_0}[ \sum_{k} ( \hat \Delta_k - \Delta_k) ],$ where $\Delta_k$ is defined in \eqref{eq:deltak}. Since the mean of $\tau_{k_j}(s,a)$ is $1/\lambda(s,a)$, it follows that given $v_k(s,a)$ after $N$ decision steps, we have
\begin{eqnarray*}
\mathbb{E}_{s_0}\left[ ( \hat \Delta_k - \Delta_k) | (v_k(s,a))_{s,a} \right] =0,
\end{eqnarray*}
which implies that $\mathbb{E}^{\bar{\mathcal{G}}}_{s_0}[  \hat \Delta_k - \Delta_k ]=0$ for each episode $k. $ Then for $N \ge 2SA,$
\begin{align*}
 \mathcal{R}_N^{\bar{\mathcal{G}}}(s_0, \mathcal{M}) \le \mathbb{E}_{s_0} \left[ \sum_{k=1}^m \Delta_k \right] &\le  4C \cdot \frac{\log(N)}{g} + \frac{\lambda_{\max}}{\lambda_{\min}} (2 + C_1).
\end{align*}
Hence, for any $N \ge 1,$
$
 \mathcal{R}_N^{\bar{\mathcal{G}}}(s_0, \mathcal{M}) \le  4C \cdot \frac{\log(N)}{g} + \frac{\lambda_{\max}}{\lambda_{\min}} (2 + C_1 + 2SA).
$
The proof of \eqref{eq:bound-RN-UCRL} is complete by setting $C' := \frac{\lambda_{\max}}{\lambda_{\min}} (2 + C_1 + 2SA).$

\subsection{Proof of Corollary~\ref{thm:UCRL2}}\label{sec:N-UB2}
\rev{We first prove \eqref{eq:bound-RN-UCRL2}. Consider a total number of decision steps of $N$, and denote $L_N:= \min\{i \ge 0: N_i >N\}$. By Theorem~\ref{thm:UCRL},
the expected regret of the algorithm incurred at each horizon $[ N_i +1, N_{i+1}]$ is at most $ \frac{4C}{g} \cdot \log (N_{i+1} - N_i) + C'$.
Then the total expected regret of the doubling  CT-UCRL algorithm over $N$ decision steps is upper bounded by
\begin{align*}
\sum_{i= -1}^{L_N - 1}  \left[ \frac{4C}{g} \cdot \log (N_{i+1} - N_i) + C' \right] & \le \frac{4C}{g}  \sum_{i=-1}^{L_N-1} \log (N_{i+1})  +   (L_N+1) \cdot C' \\
& \le \frac{4C}{g}  \sum_{j=0}^{L_N } 2^{j} \cdot \log(2)  +  (L_N+1) \cdot C' \\
& \le \frac{8C}{g}  2^{L_N}  \cdot \log(2) +  (L_N+1) \cdot C' \\
& \le \frac{16C \log(2)}{g}  \log_2 N  + [\log_2 \log_2 N +2 ] \cdot C' \\
& = \frac{16C}{g}  \log N  + [\log_2 \log_2 N +2] \cdot C' ,
\end{align*}
where in the second inequality we use the definition $N_i = 2^{2^{i}}$ for $i \ge 0$, and
the last inequality follows from the fact that $L_N \le \log_2 \log_2 N +1$.
}

\rev{
We next prove \eqref{eq: coro:ub}.
First, we obtain from \eqref{eq:TNT} that
\begin{align} \label{eq:regret-T-dist}
\mathbb{E}_{s_0}[ R_T^{\bar{\mathcal{G}}^{DT}}(s_0, \mathcal{ M})]  \le  \mathbb{E}_{s_0} [ R_{S_{N(T)+1}}^{\bar{\mathcal{G}}^{DT}}(s_0, \mathcal{ M})] + \frac{\rho^*}{\lambda_{\min}} + 1.
\end{align}
Hence it suffices to upper bound $ \mathbb{E}_{s_0}[ R_{S_{N(T) +1}}^{\bar{\mathcal{G}}}(s_0, \mathcal{ M}) ] .$ By \eqref{eq:bound-RN-UCRL2}, we have for $N \ge 1$,
\begin{align*}
 \mathcal{R}_N^{\bar{\mathcal{G}}^{DT}}(s_0, \mathcal{M}) = \mathbb{E}_{s_0}[ R_{S_{N}}^{\bar{\mathcal{G}}^{DT}}(s_0, \mathcal{ M}) ] \le   \frac{16C}{g}  \log N  + [\log_2 \log_2 N +2] \cdot C' .
\end{align*}
It follows that 
\begin{align*}
\mathbb{E}_{s_0}[ R_{S_{N(T) +1}}^{\bar{\mathcal{G}}^{DT}}(s_0, \mathcal{ M}) ] &= \mathbb{E}_{s_0}\left( \mathbb{E}_{s_0}[ R_{S_{N(T) +1}}^{\bar{\mathcal{G}}^{DT}}(s_0, \mathcal{ M}) | N(T)+1 ] \right) \\
& \le   \frac{16C}{g}  \mathbb{E}_{s_0}\left( \log (N(T)+1)  \right) + \mathbb{E}_{s_0}[\log_2 \log_2 (N(T)+1) +2] \cdot C' .
\end{align*}
As already argued in the proof of~\eqref{eq:coro-LB}, under any algorithm we have $N(t) \le A_2(t) +1$ for all $t$ and all sample paths, where $A_2(\cdot)$ is a Poisson process with rate $\lambda_{\max}$. Thus we have
$
\mathbb{E}_{s_0} \left(\log(N(T)+1) \right) \le \mathbb{E} \left(\log(A_2(T)+2) \right) \le \log(\lambda_{\max} T +2) ,
$
where $A_2(T)$ is a Poisson random variable with mean $\lambda_{\max} T,$ and the second inequality above is due to Jensen's inequality.  Similarly one can show $\mathbb{E}_{s_0}[\log_2 \log_2 (N(T)+1)] \le \log_2 \log_2 ( \lambda_{\max} T +2)$. Therefore we obtain
\begin{align*}
\mathbb{E}_{s_0}[ R_{S_{N(T) +1}}^{\bar{\mathcal{G}}^{DT}}(s_0, \mathcal{ M}) ] \le  \frac{16C}{g} \cdot \log \left( {\lambda_{\max} T +2}\right) + C' \cdot [ 2+ \log_2 \log_2 ( \lambda_{\max} T +2)].
\end{align*}
Since $r(s,a) \in [0,1]$, the long run average reward $\rho^*$ is upper bounded by $\lambda_{\max}$. Thus we can infer from  \eqref{eq:regret-T-dist} that \eqref{eq: coro:ub} holds. The proof is  complete.
}


\section{Conclusion} \label{sec:conclusion}

In this paper, we study reinforcement learning for continuous-time average-reward  Markov decision processes with unknown parameters. We establish instance-dependent logarithmic regret lower bounds that represent fundamental performance limits of learning algorithms. Meanwhile
we devise  the CT-UCRL algorithm and prove a non-asymptotic logarithmic regret upper bound. The analysis for deriving these results are substantially different from and more difficult than its discrete-time counterpart due to the presence of the random holding time at a state.

The study of regret analysis for RL in continuous time is still rare and far between.
There are many open questions. An interesting one is to extend the current setting to one with large or countably infinite state spaces.
\rev{
 Other related research directions include establishing non-asymptotic instance-dependent regret lower bounds and
designing asymptotically instance optimal algorithms with finite-time performance guarantees for average-reward CTMDPs. Beyond the setting of CTMDPs with discrete state spaces, a significant question is to study regret bounds for general controlled diffusion processes with continuous state spaces.}

\section*{Acknowledgments}
We thank the
associator editor and three anonymous referees for many constructive comments and suggestions.
Gao is supported by the Hong Kong Research Grant Council [GRF Grants 14200123, 14201520, 14201421, 14212522]. Zhou is supported by a start-up grant and the Nie Center for Intelligent Asset Management
at Columbia University. His work is also part of a Columbia-CityU/HK collaborative project that
is supported by the InnoHK Initiative, The Government of the HKSAR, and the AIFT Lab.



\bibliographystyle{siamplain}
\bibliography{RL2}

\end{document}